\documentclass{article}

\usepackage[preprint]{neurips_2026}

\usepackage[utf8]{inputenc} 
\usepackage[T1]{fontenc}    
\usepackage{hyperref}       
\usepackage{url}            
\usepackage{booktabs}       
\usepackage{amsfonts}       
\usepackage{nicefrac}       
\usepackage{microtype}      
\usepackage{xcolor}         
\usepackage{enumitem}
\usepackage{subcaption}
\usepackage{multirow}
\usepackage{multicol}
\usepackage{amsmath}
\usepackage{wrapfig}
\DeclareMathOperator*{\argmax}{arg\,max}

\usepackage{breqn}
\usepackage{xspace}

\usepackage{amsthm}
\usepackage{amsmath}
\usepackage{amssymb}
\usepackage{algorithmic}
\usepackage{algorithm}
\newtheorem{theorem}{Theorem}
\newtheorem{proposition}{Proposition}

\theoremstyle{definition}

\theoremstyle{remark}

\newcommand{\cifartwo}{CIFAR-2}
\newcommand{\cifarten}{CIFAR-10}
\newcommand{\cifarhund}{CIFAR-100}
\newcommand{\svhn}{SVHN}
\newcommand{\imagenette}{ImagenetTe}
\newcommand{\imagenet}{Imagenet}
\newcommand{\imagenettiny}{Tiny-Imagenet}

\newcommand{\vitmae}{ViT-MAE}

\newcommand{\huggingface}{HuggingFace}

\newcommand{\contrastive}{\textit{CRoPD}} 
\newcommand{\vanilla}{\textit{Vanilla}}
\newcommand{\arae}{\textit{ARAE}}
\newcommand{\vae}{\textit{VAE}}
\newcommand{\none}{\textit{Identity}}

\newcommand{\ellcon}{\ell_{\mathrm{con}}}
\newcommand{\ellsup}{\ell_{\mathrm{sup}}}
\newcommand{\Lcon}{L_{\mathrm{con}}}

\newcommand{\xadv}{x^{\mathrm{adv}}}
\newcommand{\fpre}{f_{\mathrm{pre}}}
\newcommand{\fen}{f_{\mathrm{en}}}
\newcommand{\fde}{f_{\mathrm{de}}}
\newcommand{\flast}{f_{\mathrm{last}}}

\newcommand{\adv}{\mathrm{adv}}
\newcommand{\negative}{\mathrm{neg}}
\newcommand{\cossim}{\mathrm{sim}}

\makeatletter
\DeclareRobustCommand\onedot{\futurelet\@let@token\@onedot}
\def\@onedot{\ifx\@let@token.\else.\null\fi\xspace}

\def\eg{\emph{e.g}\onedot}

\makeatother

\iftrue

\newcommand{\calvin}[1]{{\color{cyan}CH: #1}}
\newcommand{\yue}[1]{{\color{purple}Yue: #1}}
\newcommand{\mq}[1]{{\color{orange}Meiqi: #1}}
\newcommand{\eat}[1]{}

\newcommand{\TODO}[1]{\textbf{\color{red}[TODO: #1]}}
\else

\newcommand{\calvin}[1]{}
\newcommand{\yue}[1]{}
\newcommand{\mq}[1]{}
\newcommand{\eat}[1]{}

\newcommand{\TODO}[1]{}
\fi

\title{How to Enhance Downstream Adversarial Robustness (almost) without Touching the Pre-Trained Foundation Model?}

\author{%
Meiqi Liu$^*$ \\
Michigan State University\\
\texttt{liumeiqi@msu.edu}
\And
Zhuoqun Huang$^*$\\
University of Melbourne\\
\texttt{calvin.huang@unimelb.edu.au}
\And
Hongbo Li\\
Michigan State University\\
\texttt{lihongb1@msu.edu}
\And
Young-geun Kim\\
Ulsan National Institute of Science and Technology\\
\texttt{younggeun91@unist.ac.kr}
\And
Yue Xing\\
 Michigan State University\\
\texttt{xingyue1@msu.edu} 
}

\newcommand{\kim}[1]{{\color{blue}Kim: #1}}

\begin{document}

\maketitle

\begin{abstract}
  With the rise of powerful foundation models, a pre-training-fine-tuning paradigm becomes increasingly popular these days:
    A foundation model is pre-trained using a huge amount of data from various sources, and then the downstream users only need to fine-tune and adapt it to specific downstream tasks.
    However, due to the high computation complexity of adversarial training, it is not feasible to fine-tune the whole part of foundation model to improve its robustness on the downstream task.
    Observing the above challenge, we want to improve the downstream robustness almost without updating/accessing the weights in the foundation model.
    Through theoretical investigation, we identify a close relationship between robust contrastive learning with the adversarial robustness of supervised downstream tasks.
    To validate and utilize this theoretical insight, we design a simple-yet-effective robust auto-encoder as a data pre-processing method before feeding the data into the foundation model.
    The proposed approach has zero access to the foundation model when training the robust auto-encoder.
    Experiments demonstrate the effectiveness of the proposed method in improving the robustness of downstream tasks, verifying the connection between the feature robustness (implied by small adversarial contrastive loss) and the robustness of the downstream task.
\end{abstract}

\section{Introduction}

In recent years, the development of foundation models has inspired people to consider a new training paradigm:
Instead of training all layers of the neural network, the base large neural network is first trained by one party with a huge amount of data from various sources (pre-training), then the downstream users tune the last layers to adapt to downstream tasks (fine-tuning) \citep{yang2023foundation}.

Meanwhile, although recent advances in deep learning and machine learning have led to breakthrough performance and have been widely applied in practice, both empirical evidence (e.g., \citep{madry2017towards}) and theoretical investigations (e.g., \citep{haldar2024effect}) reveal that deep learning models can be fragile and vulnerable against adversarial input which is intentionally perturbed to mislead the model.
Adversarial training is one of the most popular remedies to improve the robustness of these models \citep{madry2017towards}.

With the new pre-training-fine-tuning paradigm, many studies consider improving adversarial robustness in the downstream task using adversarial pre-training and clean fine-tuning.
For example, some works empirically observe the robustness of a robust pre-trained model using robust contrastive learning being inherited to  downstream tasks \citep{shafahi2019adversarially,salman2020adversarially,deng2021adversarial,zhang2021pre,kim2020adversarial,fan2021does}.

However, although the rise of the pre-training-fine-tuning paradigm provides one possible way to reduce the computation cost for the downstream users, concerns have been raised regarding the computational cost of adversarial training: Compared to clean training, the cost of adversarial training is much higher since the attacks are recalculated in each training iteration.
For example, while clean training of ResNet18 for CIFAR-10 takes 1 hour on a single NVIDIA V100 GPU, adversarial training can take more than 20 hours \citep{rice2020overfitting}.
Consequently, in the pre-training-fine-tuning paradigm, though pre-training parties are mostly resource-abundant entities like OpenAI,
adversarial training can still be burdensome as the computation cost to clean pre-train a GPT-3 alone is already estimated to be up to \$$4.6$m \citep{cost2024}.
    The infeasible adversarial training cost leaves an open question to be answered:
    \begin{center}
        \textit{How can we leverage adversarial training in foundation models while maintaining a low computational cost?}
    \end{center}

    \begin{figure*}[t]
        \centering
        \includegraphics[scale=0.18]{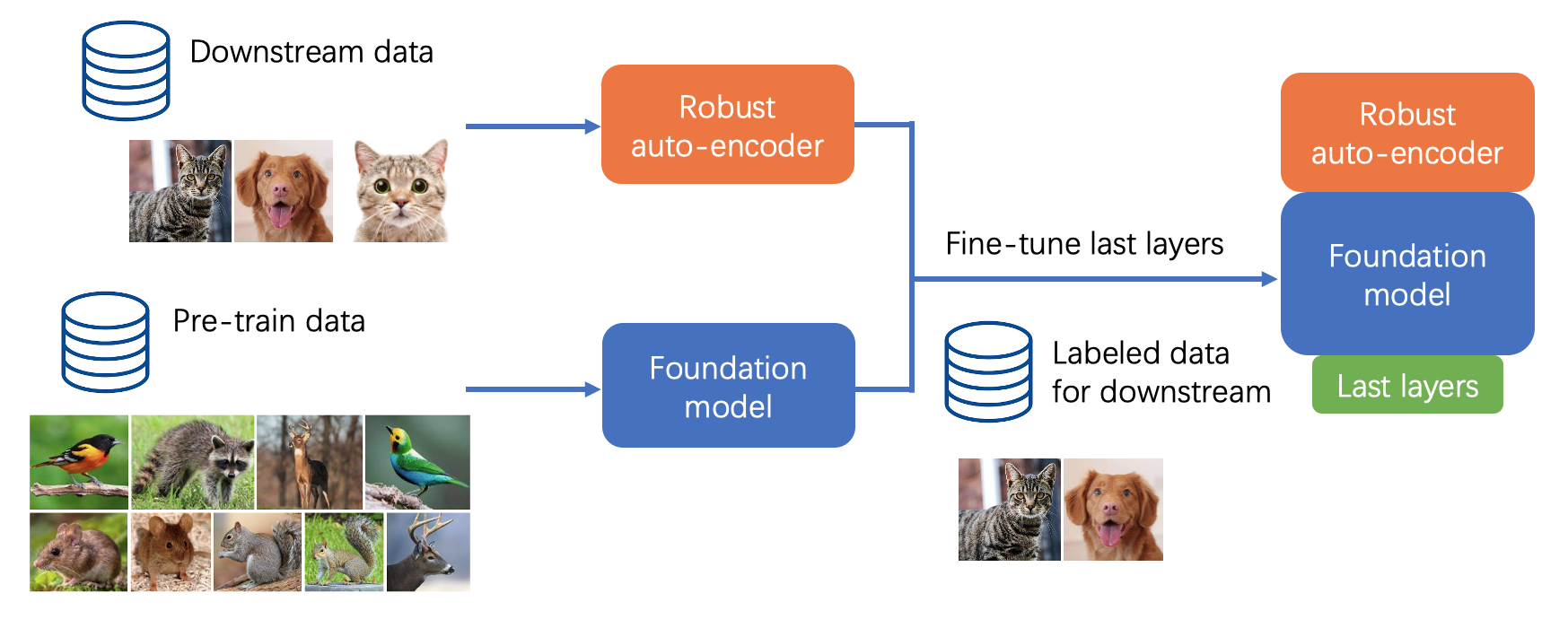}\vspace{-0.1in}
        \caption{Pre-processing downstream data with a robust auto-encoder. While freezing the pre-trained foundation model, we use adversarial training to train a robust auto-encoder via leveraging adversarial contrastive loss. After that, the trained robust auto-encoder is used to pre-process downstream data. These pre-processed inputs are then fed into the foundation model for fine-tuning and final deployment for downstream tasks. }
        \label{fig:adv_contra}
    \end{figure*}

    To address the above challenge, we provide a theoretical analysis to bound the downstream adversarial loss using adversarial contrastive loss. 
    Further, based on the theoretical insights, we design a robust pre-processor to validate the importance of feature robustness, and pursue robustness in the downstream tasks with only clean-trained foundation models.

    Our main contributions are summarized as follows:
    \begin{itemize} [leftmargin=1em, itemsep=0.3em, topsep=0pt]
        \item We provide a theoretical analysis of the downstream adversarial loss. Based on our derivation, for classification tasks with cross-entropy loss, the downstream adversarial loss can be upper-bounded by a combination of the downstream clean loss and the adversarial contrastive loss (Theorem \ref{thm:downstream}).
        \item Furthermore, a simple alternative to the pre-training-fine-tuning system that uses a vanilla, clean-trained auto-encoder may suffer from non-negligible adversarial loss (Proposition \ref{pro:adv loss}).
        \item Inspired from the above theoretical insights, we consider the following robust data pre-processor as in Figure~\ref{fig:adv_contra}: We train an auto-encoder with adversarial contrastive loss to obtain a robust auto-encoder as a data pre-processor and feed the corrupted data to the pre-processor first and then feed the foundation model with the pre-processed data.

              Specifically, the robust auto-encoder consists of an encoder-decoder structure, which is trained in an unsupervised manner without requiring labels. The pre-processed output is then passed to the foundation model for downstream tasks. The training of the auto-encoder is guided by adversarial contrastive loss, ensuring that the latent representations are robust against adversarial perturbations while maintaining the unsupervised nature of the process.
              
              While existing literature, e.g., \citep{salehi2021arae,zhou2023eliminating}, attempts to add adversarial training, it needs to access the foundation model in adversarial training. In contrast, we do not access the foundation model when training the robust pre-processor.
              We name our simple-yet-effective approach as  \textbf{C}ontrastive \textbf{Ro}bust \textbf{P}reprocessing \textbf{D}efense (\contrastive).
              Empirical experiments are conducted to verify the effectiveness of \contrastive.

    \end{itemize}

\section{Related Works}
This section lists related works in the field of adversarial training and contrastive learning.

\noindent\textbf{Adversarially Robust Pre-processing.}
In literature, some existing works,  e.g., \citep{sahay2019combatting,zhou2021improving,cann2022robust,zhou2023eliminating},
consider implementing robust data pre-processors to defend against adversarial attacks.
Compared to the existing literature, our proposed method avoids tuning foundation models when training the robust data pre-processor.
In contrast, these existing works use the supervised learning loss to train the pre-processor, thus heavily relying on the pre-trained model. Another work \citet{salehi2021arae} proposes to design an attack to corrupt the latent space when training the auto-encoder to improve the output quality against adversarial attacks.

There is also other literature that designs attacks for auto-encoders, e.g., \citep{tabacof2016adversarial}.
\citet{tabacof2016adversarial} designs an auto-encoder that receives an input image of one class but outputs another
image similar to the input but belongs to another class.


\noindent\textbf{Adversarial Training.}
There are fruitful studies in the area of adversarial training. For methodology, there are many techniques, e.g., \citep{goodfellow2014explaining,zhang2019theoretically,wang2019improving,cai2018curriculum,zhang2020attacks,carmon2019unlabeled,gowal2021improving,mo2022adversarial,wang2022improving}.
Theoretical investigations have also been conducted for adversarial training from different perspectives.
For instance, \citet{chen2020more, javanmard2020precise,taheri2021statistical,yin2018rademacher,raghunathan2019adversarial,najafi2019robustness,min2020curious,hendrycks2019using,dan2020sharp,wu2020revisiting,javanmard2021adversarial,deng2021improving,javanmard2022precise} study the statistical properties of adversarial training. \citet{sinha2018certifying,wang2019convergence,xiao2022stability,xiao2022adaptive} study the optimization aspect of adversarial training. Lastly, \citet{zhang2020over,wu2020does} work on theoretical issues related to adversarial training with deep learning.

\noindent\textbf{Contrastive Learning.}
Contrastive learning is a popular self-supervised learning algorithm.
It uses unlabeled images to train representations that distinguish different images invariant to non-semantic transformations \citep{mikolov2013distributed,oord2018representation,arora2019theoretical,dai2017contrastive,chen2020simple,tian2020makes,chen2020simple,khosla2020supervised,haochen2021provable,chuang2020debiased,xiao2020should,li2020prototypical}.
Besides empirical studies, there are also many theoretical studies, e.g., \citep{saunshi2019theoretical,haochen2021provable,haochen2022beyond,shen2022connect,haochen2022theoretical,saunshi2022understanding}.
Based on both empirical and theoretical studies, a common understanding is that contrastive learning can capture the features from the input via comparing positive (similar) and negative (dissimilar) pairs.

    \section{Preliminaries}\label{sec:prelim}
    In the following, we introduce the essential notations and definitions for the data and the algorithm.

\textbf{Data Distribution.} We consider classification in the downstream task. Given a dataset \( \mathcal{D} = \{(x_i, y_i)\}_{i=1}^n \), where \( x_i \) is the input feature and \( y_i \in \mathcal{Y} \) is the label in label space $\mathcal{Y}$, we assume that these data points are drawn from an unknown distribution \( q(x, y) \).

\textbf{Adversarial Attack.} The adversarial attack will be used for both the supervised adversarial training and the self-supervised contrastive learning. Given a loss function $\ell(\cdot,\cdot)$ and a model $f(\cdot)$,
an adversarial attack aims to figure out the worst-case perturbed features which maximizes the loss, i.e.,
\begin{eqnarray*}
    \xadv \triangleq \argmax_{ \bar{x}\in\mathcal{A}_{\epsilon}(x)} \ell(f(\bar{x}),\cdot).
\end{eqnarray*}
The second argument of $\ell(\cdot,\cdot)$ can be either the label for supervised learning, or $f(x')$ for some other data $x'$ in contrastive learning. The notation \( \mathcal{A}_{\epsilon}(x) \), referred to as the threat model, defines the set of permissible adversarial perturbations for the input \( x \).
Specifically,
    $\mathcal{A}_{\epsilon}(x) = \{ \bar{x} : \| \bar{x} - x \|_p \leq \epsilon \},$
where \( p \)-norm (\( p = 2 \) or \( \infty \)) determines the metric used, and \( \epsilon \) specifies the attack budget. For example, when \( p = \infty \), \( \mathcal{A}_{\epsilon}(x) \) forms an \( \mathcal{L}_\infty \)-ball around \( x \) with radius \( \epsilon \), i.e., the pixel attack in image data.

    \textbf{Adversarial Training.} Adversarial training is a generic algorithm to iteratively train the neural network: In each iteration, given the current model $f$, we first calculate the attacked version of the training data using the above definition of adversarial attack, and then use the attacked data to calculate the loss and update the model correspondingly.

For clean training and adversarial training in this paper, we use ``clean training" to minimize the clean loss and use ``adversarial training" to minimize the adversarial loss.

\textbf{Contrastive Learning and Robustness Inheritance.} While the autoencoder-based pre-processor learns representations through reconstruction, it does not explicitly enforce semantic consistency or robustness in the latent space. Therefore, we leverage contrastive learning in the pre-processor.

Different from supervised learning, contrastive learning does not need human-annotated labels/responses. The aim is to figure out an encoder \( \fen(\cdot) \) so that the similarity between similar pairs (positive pairs) of data is maximized, while the similarity between dissimilar pairs (negative pairs) is minimized.
We denote the output of the encoder by \( z = \fen(x) \) for notational simplicity where appropriate.
To theoretically connect contrastive learning with supervised learning, \citet{arora2019theoretical} formalize the semantic similarity using latent classes and prove that minimizing the contrastive loss leads to a representation function that achieves a low average linear classification loss on downstream tasks.
They establish that under certain conditions, the contrastive loss serves as an upper bound of the expected supervised loss, theoretically supporting the effectiveness of contrastive learning in downstream tasks.

Extending from the clean contrastive learning, the adversarial contrastive learning trains an encoder \( \fen(\cdot) \). Empirically, the adversarial contrastive loss considers \( n \) samples from the dataset:
\begin{equation*}
    \widehat{L}_{\text{con}}(\fen) = \frac{1}{n} \sum_{i=1}^{n} \ellcon(\fen(\xadv_i), \fen(x_i))).
\end{equation*}
The loss $\ellcon$
is defined as
\begin{eqnarray*}
     \ellcon(\fen(\xadv_i), \fen(x_i)) 
    &=& -\log \frac{\exp(\cossim(\fen(\xadv_i), \fen(x_i))/\tau)}{\sum_{x^{\negative} \in X^{\negative} } \exp(\cossim(\fen(x_i), \fen(x^{\negative}))/\tau)},
\end{eqnarray*}
where \( \text{sim}(\cdot, \cdot) \) is the cosine similarity between two latent vectors as
$
    \cossim(\mathbf{u}, \mathbf{v}) = {\mathbf{u}^T \mathbf{v}}/{(\|\mathbf{u}\|_2 \|\mathbf{v}\|_2)}.
$
Given the above, one can also extend the definition to the population version \( {L}_{\text{con}} \) which is
\begin{equation}
    \Lcon(\fen ) = \mathbb{E}_{q(x)}\left[ \ellcon(\fen(\xadv), \fen(x)) \right],\label{eqn:Lcon}
\end{equation}


The above formulation allows for adversarial examples to be integrated into the loss by maximizing the similarity between clean and adversarial representations while minimizing similarities with negative samples.

To construct the dissimilar set $X^{\negative}$ in the above, we first take $X^{\negative}$ as a set of clean examples $X = \{x_1, \ldots x_M\}$ to calculate the corresponding $\xadv_j$'s as the approximation to form a set $X^{\text{adv}}$, and then take $X^{\negative}=X \bigcup X^{\adv} \backslash \{x_i, \xadv_i\}$ when figuring out the $\xadv_i$ on the numerator of $\ellcon$.

Regarding the robustness inheritance phenomenon in adversarial contrastive learning, the common scenario is to use adversarial training to train a contrastive model and then use clean training to replace the projector (the last linear layer) to adapt to the downstream task. In this way, the downstream robustness is significantly better than training the downstream task from scratch \citep{kim2020adversarial}.


\section{Robust Data Pre-Processing}

We first present the theoretical analysis in Section \ref{sec:main_theory} and the potential issue in auto-encoder in Section \ref{sec:existing_issue}, and then introduce the algorithm in Section \ref{sec:robust_pre_process}.

\subsection{Main Theory}\label{sec:main_theory}

To analyze the robustness of the downstream task $T$ with  the distribution \( q(y \mid x) \), the cross-entropy loss (i.e., $\ellsup$) given the encoder parameter $\theta$ is
\begin{equation*}
    \mathbb{E}_{q(x,y)} \left[ \mathbb{E}_{p_\theta(z \mid x)}[-\log \hat{p}_T(y \mid z)] \right],
\end{equation*}
and the corresponding adversarial loss is
\begin{equation*}
    \mathbb{E}_{q(x,y)} \left[\max_{\xadv \in \mathcal{A}_{\epsilon}(x)} - \log \mathbb{E}_{p_\theta(z \mid \xadv)}[\hat{p}_T(y \mid z)]\right],
\end{equation*}
where $\hat{p}_T$ is the model's output probability of label $y$.
We aim to learn a classifier \( \hat{p}_T(y \mid \fde(\fen(x))) \) that predicts \( y \) based on the features \( \fde(\fen(x)) \).

The following theorem describes how the adversarial attack impacts the downstream classification performance, and how it is related to the adversarial contrastive loss:

\begin{theorem}\label{thm:downstream}
Given encoder \( \fen(\cdot) \), and the downstream prediction function satisfies the following Lipschitz continuity condition: there exists a constant \( M_C > 0 \) such that, for all \( y \in \mathcal{Y} \), the mapping
\(
z \;\mapsto\; -\log \hat{p}_T(y \mid z)
\)
is \( M_C \)-Lipschitz continuous.

Then, there exists a constant \(\kappa > 0\), dependent on the Lipschitz constant \(M_C\) and the properties of the adversarial contrastive loss, such that the following general bound holds:
\begin{eqnarray*}
    \mathbb{E}_{q(x,y)}\left[\,\max_{\xadv \in \mathcal{A}_{\epsilon}(x)} -\log \hat{p}_T(y \mid \fen(\xadv))\,\right] 
    &\leq& \mathbb{E}_{q(x,y)}\left[-\log \hat{p}_T(y \mid \fen(x))\right] + \kappa \cdot \Lcon(\fen).
\end{eqnarray*}

In particular, when \(\Lcon(\fen)\) is chosen as the adversarial contrastive loss defined in (\ref{eqn:Lcon}), the constant \(\kappa\) can be explicitly expressed in terms of the Lipschitz constant \(M_C\), the temperature parameter \(\tau\), and the number of negative samples \(M\). Under this choice, the difference between the adversarial and clean downstream prediction losses can be explicitly bounded by
\begin{eqnarray*}
    &&\mathbb{E}_{q(x,y)}\left[\,\max_{\xadv \in \mathcal{A}_{\epsilon}(x)} -\log \hat{p}_T(y \mid \fen(\xadv))\,\right]
    -\mathbb{E}_{q(x,y)}\left[-\log \hat{p}_T(y \mid \fen(x))\right] \notag\\
&\leq& M_C \sqrt{\,2 - 2\,\tau\left(\ln(M+1)-\Lcon(\fen)\right)}.
\end{eqnarray*}
This theorem implies that minimizing the adversarial contrastive loss \(\Lcon(\fen)\) effectively controls and bounds the adversarial prediction loss of the downstream task.
\end{theorem}

The proof of Theorem \ref{thm:downstream} is in Appendix \ref{sec:appendix:proof}. Theorem~\ref{thm:downstream} illustrates how the downstream adversarial loss can be connected to the adversarial contrastive loss: it is upper bounded by the downstream clean loss plus the adversarial contrastive loss. This implies that, if a data pre-processor can achieve a small adversarial contrastive loss, it can also lead to the robustness of the downstream classification.

\subsection{Potential Issue in Adversarial Reconstruction Loss}\label{sec:existing_issue}

To highlight the importance of the feature robustness (i.e., a small adversarial contrastive loss), when there is no data pre-processing, or we only attack on the reconstruction loss to train the auto-encoder (the benchmark \arae), we further show that these scenarios can result in a poor downstream robustness. The following proposition provides an example:

\begin{proposition}\label{pro:adv loss}
    There exists a data generation model $(x, y)$, an auto-encoder $(\fen, \fde)$, and a classifier $(\fpre, \flast)$ such that if the  clean reconstruction loss
    \(\mathbb{E}_{q(x)} \left[ \| \fde(\fen(x)) - x \|^2 \right]
    \) and the  adversarial reconstruction loss \(\mathbb{E}_{q(x)} \left[ \| \fde(\fen(\xadv)) - x \|^2 \right]\) for adversarial examples $\xadv$ within a perturbation budget $\epsilon$
    are 0 and $O(1/n)$ respectively, then the adversarial classification loss satisfies
        {
            \begin{align*}
                 & \mathbb{E}_{q(x,y)} \left[ \max_{\xadv \in \mathcal{A}_{\epsilon}(x)} -\log \hat{p}_T\left(y \mid \flast(\fpre(\fde(\fen(\xadv))))\right) \right] \notag \\
                 & \geq \Gamma> \mathbb{E}_{q(x,y)} \left[ -\log \hat{p}_T\left(y \mid \flast(\fpre(x))\right) \right] = O(\delta),
            \end{align*}
        }
    where $\Gamma=O(-\log(\delta))$ is a large value, and $\delta=o(1/n)$ associated with the data distribution.
\end{proposition}

Proposition \ref{pro:adv loss} is a possible scenario where the auto-encoder is trained to achieve a small adversarial reconstruction loss while the downstream classification task has a poor robustness. It constructs a discrete data distribution with well-separated points and a classifier that assigns high probability to clean inputs and low probability to perturbed inputs. The detailed proof can be found in Appendix \ref{sec:appendix:proof}.

Proposition \ref{pro:adv loss} underscores the need for a better approach for robust auto-encoder.
Further inspired by Theorem \ref{thm:downstream}, by integrating adversarial contrastive loss when training the pre-processors, the robustness issue in auto-encoder can be mitigated. We present the algorithm in the next subsection.

\subsection{Practical Algorithm}\label{sec:robust_pre_process}

To design the robust data pre-processor, we leverage adversarial contrastive learning when training the robust auto-encoder.
Unlike the original contrastive learning framework, we assume that the foundation model already exists and try to avoid accessing it to reduce computational costs.
In this case, since the foundation model receives images as inputs, we need to develop a model to pre-process the images before feeding them into the foundation model. Consequently, we leverage contrastive learning to train an auto-encoder, the latter of which outputs an image. The proposed algorithm is summarized in Algorithm \ref{alg:contra}, and the graphical illustration can be found in Figure~\ref{fig:adv_contra}.

\begin{algorithm*}
    \caption{\textbf{C}ontrastive \textbf{Ro}bust \textbf{P}reprocessing \textbf{D}efense (CRoPD)}\label{alg:contra}
    \begin{algorithmic}[1]
        \STATE Use pre-training dataset $\mathcal{D}_{pretrain}$ to train a  foundation model $f_{last}(f_{pre}(\cdot))$.
        \STATE Use the downstream dataset $\mathcal{D}_{down}$ to train a robust auto-encoder via minimizing the loss
        \begin{eqnarray*}
            \min_{f_{en},f_{de}}\sum_{x\in\mathcal{D}_{down}} \|f_{de}(f_{en}(x))-x\|^2+\lambda \sup_{\xadv\in\mathcal{A}_{\epsilon}(x)} \ellcon(f_{en}(\xadv),f_{en}(x)),
        \end{eqnarray*}
        \STATE Use the labeled downstream dataset $\mathcal{D}_{label}$ to adjust the last layers $f_{last}$ for the downstream task:
        $$\min_{f_{last}} \sum_{(x,y)\in \mathcal{D}_{label}} L_{sup}( f_{last}(f_{pre}(f_{de}(f_{en}(x)))) , y).$$
    \end{algorithmic}
\end{algorithm*}

There are several components in whole framework:

First, for the robust auto-encoder, assume we have a robust auto-encoder defined by an encoder \( f_{en}(x) \) and a decoder \( f_{de}(z) \), where \( z = f_{en}(x) \) represents the latent representation of the input feature \( x \). This auto-encoder is trained using a combination of reconstruction loss and adversarial contrastive loss. The reconstruction loss, $\|f_{de}(f_{en}(x))-x\|^2$ ensures that the decoder \( f_{de} \) can accurately reconstruct the original input from the encoded features. The adversarial contrastive loss, $\ell_{\text{con}}(f_{en}(\xadv),f_{en}(x))$  promotes robustness of the features \( f_{en}(x) \) against \( \xadv \).

Second, in addition to the robust auto-encoder, since the output format of the foundation model might be different from the downstream task, we further train a new last layer on top of the foundation model. Recall that the output of the robust auto-encoder is $f_{de}(f_{en}(x))$, the output of the foundation model then becomes $f_{pre}(f_{de}(f_{en}(x)))$. After passing this to the last linear layer, we get the output as $f_{last}(f_{pre}(f_{de}(f_{en}(x))))$, and we minimize the downstream loss, denoted as $L_{sup}$.

To connect with Theorem \ref{thm:downstream}, in \contrastive, since we leverage adversarial contrastive loss in training the auto-encoder, its value is well controlled. In contrast, similar to Proposition \ref{pro:adv loss}, when there is no data pre-processing, or we only attack on the reconstruction loss to train the auto-encoder, there is no anticipation on how the adversarial contrastive loss behaves, and the model robustness can be poor.

\section{Experiments}\label{sec:exp}
In the experiments, we aim to demonstrate the effectiveness of the proposed robust pre-processor method. The expected result is that, \contrastive\ leads to robustness much stronger than using a non-robust-contrastive-learning-based data pre-processor, highlighting the importance of feature robustness (a small adversarial training loss). In addition, since the robust pre-processor is a small model, the final adversarial robustness of the downstream task may be a bit worse than using adversarial training to fine-tune the foundation model. However, the computational cost of \contrastive\ is much smaller than fine-tuning a foundation model using adversarial training. Due to page limit, we postpone the some additional experiments to Section \ref{sec:transfer} and \ref{app-sec:experiment-results}. 

\subsection{Experimental Setups}
Below we summarize the setups of our experiments.

\textbf{Datasets.}
We conduct experiments on \cifarten, \cifarhund\ \citep{krizhevsky2009learning}, \svhn\ \citep{yuval2011reading}, \imagenette\ \citep{imagenette} (a subset of $10$ classes from \imagenet\ \citep{imagenet09}), and \imagenettiny\ \cite{le2015tiny}.
We also consider a special variation of \cifarten, dubbed \cifartwo, that subsets the first two classes of \cifarten\ (\textit{airplane} and \textit{automobile}) for computationally intensive robust training experiments.
We use the original dataset split to train and evaluate our models.
We briefly describe these datasets below: \textbf{\cifarten}: This dataset consists of $60\,000$ -- $32 \times 32$ color images across $10$ categories. \textbf{\cifartwo}: A binary subset of \cifarten, containing only the first two classes (airplane and automobile). \textbf{\cifarhund}: This dataset contains $60\,000$ -- $32 \times 32$ color images with $100$ categories. \textbf{\svhn}: This dataset contains $630\,420$ -- $32 \times 32$ color images of digits ($0$-$9$) cropped from house numbers in Google Street View images. \textbf{\imagenette}: A subset of $13\,000$ images of $10$ classes from the ImageNet dataset \citep{imagenet09}. \textbf{\imagenettiny}: This dataset contains $100\,000$ images of 200 classes.

\begin{table*}
    \centering
\resizebox{\linewidth}{!}{
    \begin{tabular}{l@{}c||c@{\hskip 6pt}c@{\hskip 6pt}c@{\hskip 6pt}c|c@{\hskip 6pt}c@{\hskip 6pt}c@{\hskip 6pt}c}
        \toprule
        Pre-processor               & Fine-tuning & \multicolumn{4}{c|}{\textbf{Clean}} & \multicolumn{4}{c}{\textbf{Robust}}                                                                                                                                                                                                                       \\
                                    & foundation  & Natural                             & PGD-10                               & PGD-20                               & AutoAttack*                           & Natural        & PGD-10                               & PGD-20                               & AutoAttack*                          \\
        \midrule
        \midrule
        \none                       &             & \textbf{99.54}                      & 32.78{\scriptsize$\pm$2.15}          & 19.14{\scriptsize$\pm$1.78}          & \phantom{0}2.11{\scriptsize$\pm$0.63} & \textbf{97.28} & 73.46{\scriptsize$\pm$1.96}          & 65.76{\scriptsize$\pm$2.11}          & 45.28{\scriptsize$\pm$2.14}          \\
        \quad + LoRA                &             & 98.44                               & 43.11{\scriptsize$\pm$2.20}          & 33.55{\scriptsize$\pm$2.09}          & \phantom{0}8.38{\scriptsize$\pm$1.23} & 97.54          & 65.50{\scriptsize$\pm$2.08}          & 55.96{\scriptsize$\pm$2.16}          & 28.18{\scriptsize$\pm$2.04}          \\
        \vae                        &             & 94.71                               & 57.23{\scriptsize$\pm$2.18}          & 46.45{\scriptsize$\pm$2.20}          & 34.78{\scriptsize$\pm$2.04}           & 92.80          & 71.47{\scriptsize$\pm$2.03}          & 62.54{\scriptsize$\pm$2.14}          & 56.03{\scriptsize$\pm$2.15}          \\
        \vanilla                    &             & 99.31                               & 34.10{\scriptsize$\pm$2.03}          & 22.09{\scriptsize$\pm$1.77}          & \phantom{0}7.87{\scriptsize$\pm$1.20} & 97.05          & 73.92{\scriptsize$\pm$1.96}          & 66.84{\scriptsize$\pm$2.07}          & 55.45{\scriptsize$\pm$2.11}          \\
        \arae                       &             & 99.20                               & 46.84{\scriptsize$\pm$2.17}          & 36.22{\scriptsize$\pm$2.08}          & 70.62{\scriptsize$\pm$1.98}           & 97.29          & 78.83{\scriptsize$\pm$1.79}          & 73.16{\scriptsize$\pm$1.93}          & 81.78{\scriptsize$\pm$1.67}          \\
        \midrule
        \contrastive, $\lambda=0.1$ &             & 99.00                               & 45.57{\scriptsize$\pm$2.12}          & 34.65{\scriptsize$\pm$1.94}          & 19.48{\scriptsize$\pm$1.76}           & 97.30          & 77.07{\scriptsize$\pm$1.85}          & 71.10{\scriptsize$\pm$1.99}          & 63.90{\scriptsize$\pm$2.13}          \\
        \contrastive, $\lambda=1$   &             & 96.81                               & 82.04{\scriptsize$\pm$1.69}          & 77.42{\scriptsize$\pm$1.83}          & 87.06{\scriptsize$\pm$1.48}           & 96.31          & 88.00{\scriptsize$\pm$1.45}          & 85.64{\scriptsize$\pm$1.55}          & 91.80{\scriptsize$\pm$1.24}          \\
        \contrastive, $\lambda=10$  &             & 96.66                               & \textbf{84.60{\scriptsize$\pm$1.60}} & \textbf{80.28{\scriptsize$\pm$1.73}} & \textbf{90.56{\scriptsize$\pm$1.30}}  & 96.41          & \textbf{89.04{\scriptsize$\pm$1.38}} & \textbf{87.03{\scriptsize$\pm$1.44}} & \textbf{93.39{\scriptsize$\pm$1.07}} \\
        \midrule
        \none                       & \checkmark  & 99.40                               & 70.13{\scriptsize$\pm$2.02}          & 64.45{\scriptsize$\pm$2.17}          & 21.02{\scriptsize$\pm$1.83}           & 99.30          & 95.96{\scriptsize$\pm$0.87}          & 95.39{\scriptsize$\pm$0.92}          & 96.41{\scriptsize$\pm$0.83}          \\
        \bottomrule
    \end{tabular}}
    \caption{
        Clean and robust performance comparison of different pre-processors on \cifartwo.
        The results are presented for both clean and robust trained (Robust) models,
        evaluated under clean conditions and against adversarial attacks (PGD-10, PGD-20, and AutoAttack).
        The best values in each column, excluding full-rank fine-tuned models, are highlighted.
        \contrastive\ achieves the best robust accuracies for both PGD and AutoAttack, beating the runner-up by $11\%$ to $27\%$.
        \footnotesize * We use APGD-CE, FAB, Square Attack for CIFAR-2,
        because the other attacks are not compatible with the binary classification task (require at least 4 classes).
    }\vspace{-0.2in}
    \label{tab:cifartwo-all-results}
\end{table*}

\begin{wrapfigure}[]{r}{0.5\columnwidth}
    \centering
    \vspace{-0.2in}
    \small
    \begin{tabular}{l@{\hspace{3pt}}c@{\hspace{3pt}}c@{\hspace{3pt}}c@{\hspace{3pt}}c}
        & \cifarten & \cifarhund & \svhn & \imagenette \\
        Orig                                                                                  &
        \includegraphics[width=0.095\textwidth]{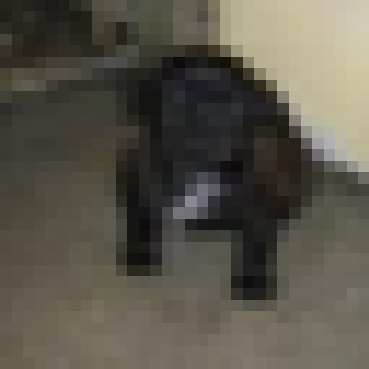}          &
        \includegraphics[width=0.095\textwidth]{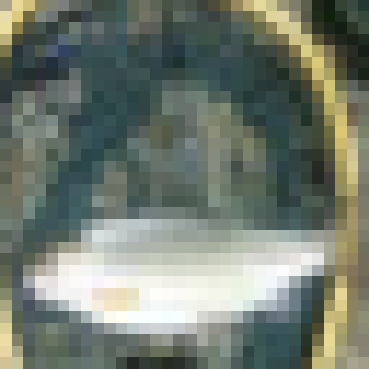}         &
        \includegraphics[width=0.095\textwidth]{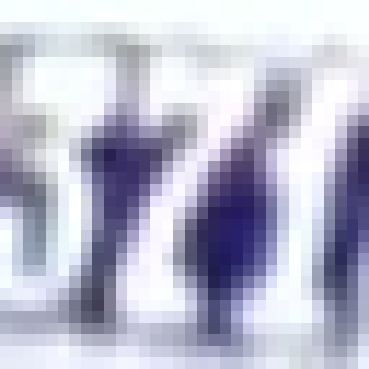}              &
        \includegraphics[width=0.095\textwidth]{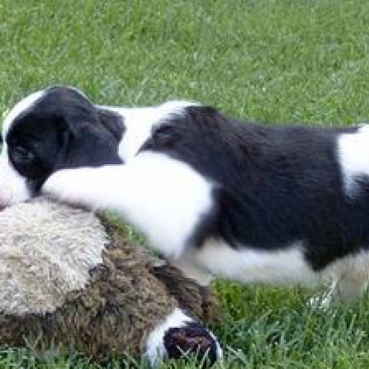}                                                       \\
        \arae                                                                                 &
        \includegraphics[width=0.095\textwidth]{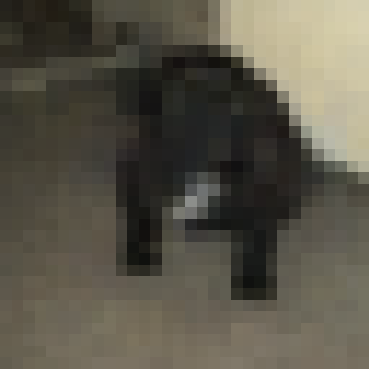}  &
        \includegraphics[width=0.095\textwidth]{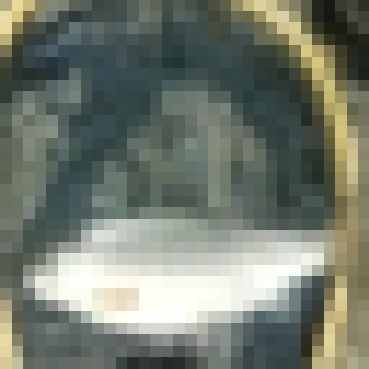} &
        \includegraphics[width=0.095\textwidth]{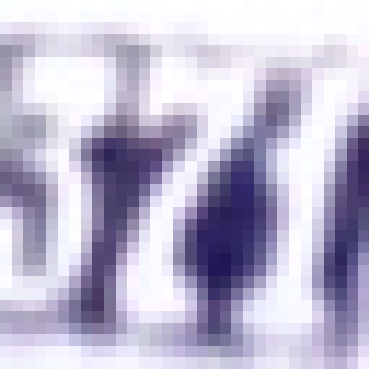}      &
        \includegraphics[width=0.095\textwidth]{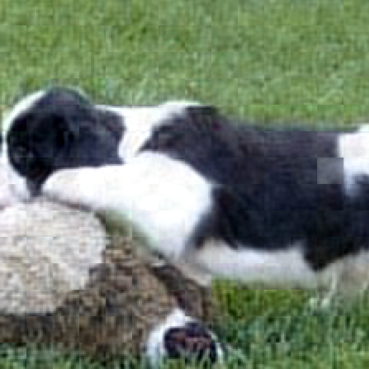}                                               \\
        \contrastive                                                                          &
        \includegraphics[width=0.095\textwidth]{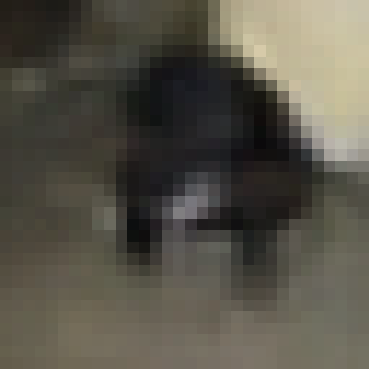}       &
        \includegraphics[width=0.095\textwidth]{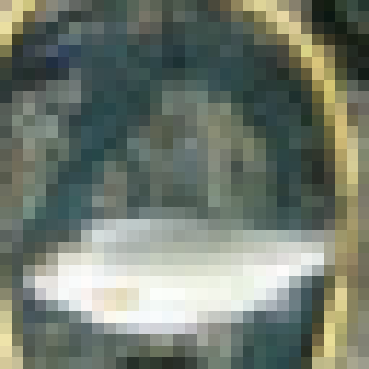}      &
        \includegraphics[width=0.095\textwidth]{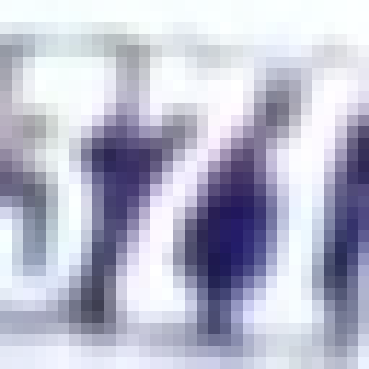}           &
        \includegraphics[width=0.095\textwidth]{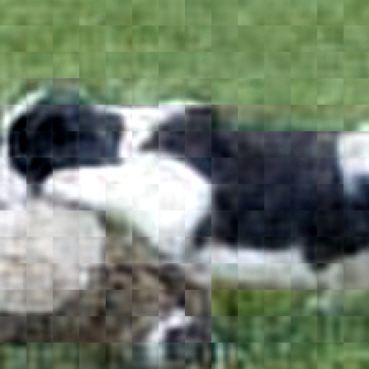}                                                    \\
    \end{tabular}
    \caption{
        Sample image reconstructions.
        Top row: original images, middle row: \arae\ reconstructions, bottom row: \contrastive\ reconstructions.
        Columns correspond to different datasets.
        \arae\ reconstructions are sharper as expected, while \contrastive\ reconstructions are purified and more
        robust for downstream tasks.
    }
    \label{fig:auto-encoder-recon}
\end{wrapfigure}
\textbf{Pre-processors.}
Following \citet{zhou2023eliminating}, we use a variant of \vitmae\ architecture that utilizes $50\%$ deterministic masking for consistent reconstruction.
The details is in Appendix~\ref{app-sec:experiment-setup}. In addition to \vanilla, which utilizes an auto-encoder trained with reconstruction loss only, we also include two other pre-processor baselines, \arae\ (adversarial reconstruction loss) \citep{salehi2021arae}
and pre-trained \vae\ \citep{kingma2014autoencoding} from \huggingface\ Diffusers \citep{platen2022diffusers}.
For \arae, although it utilizes adversarial training to train an auto-encoder,
the main purpose is to improve the output quality of the auto-encoder rather than distinguishing similar and dissimilar data.
We also include \vae\ because its denoising capabilities may also mitigate adversarial attacks.
Finally, \none\ represents the case without any pre-processor and is the most vulnerable baseline.

Following observations in \citet{chen2020simple}, though \vitmae\ naturally outputs latent embeddings,
it is ineffective for \contrastive\ or \arae\ if we naively use this embedding to align the latents.
As a result, we use pooling and a two-layered projector to reduce the latent to a $128$-dimension vector for training \contrastive\ and \arae.

Before performing the downstream task,
we first train the \contrastive, \vanilla, and \arae\ using the downstream dataset to obtain pre-processors.
Some sample image reconstructions by \arae\ and \contrastive\ are demonstrated in Figure~\ref{fig:auto-encoder-recon}.

\textbf{Foundation Model and Downstream Task.}
For all experiments, we use the \huggingface\ Transformers \citep{wolf2020transformers} to load a
pre-trained large \vitmae\ \citep{he2021masked} as the foundation model.
Due to the different characteristics of the datasets, we consider different ways of using the foundation model, and postpone the details to Appendix \ref{app-sec:experiment-setup}. Finally, to perform classification, we train a linear layer that maps features of the foundation model to labels.

\textbf{Attacks.}
We evaluate the robustness of the system using a true white-box attack scenario and PGD with two settings (PGD-10 and PGD-20). Adversarial examples are dynamically generated during evaluation by calculating gradients through the entire pipeline, which includes the robust auto-encoder (encoder and decoder), the foundation model, and the linear classifier. Specifically, we use an \( \mathcal{L}_\infty \) attack with a maximum perturbation strength \( \epsilon = 8/255 \) for CIFAR-2, CIFAR-10, and CIFAR-100, and \( \epsilon = 4/255 \) for ImagenetTe, consistent with the training setup. Unlike training, where FGSM is applied, evaluation employs iterative PGD attacks to generate more challenging adversarial samples. For PGD-10, the attack is iterated for 10 steps with each step size limited to \( \epsilon / 5 \). Similarly, PGD-20 performs 20 iterations with each step size limited to \( \epsilon / 10 \).

\subsection{Binary Classification}
Since \cifartwo\ contains only $10\,000$ training samples and the computation is much less expensive than the other datasets,
we benchmark the performance of all methods on \cifartwo\ with both clean and adversarial training.
This includes the proposed method \contrastive, other pre-processors (\vanilla, \arae, \vae), no pre-processor (\none),
and a benchmark of fine-tuning the whole foundation model using clean and adversarial training.
Among all settings, robust fine-tuning is far more expensive than clean training, up to 10 times slower.
The slow convergence rate of robust fine-tuning makes it the slowest approach.

The results are summarized in Table~\ref{tab:cifartwo-all-results}.
In general, \contrastive\ outperforms the other methods when trained with either clean and adversarial training.
The following presents some details:

First, as mentioned in the experiment setup above, we provide a benchmark of robust fine-tuning for the foundation model.
From Table~\ref{tab:cifartwo-all-results}, it achieves a robust test accuracy of around 96\% for both PGD and AutoAttack,
which is the highest compared to others.


Second, comparing (\none, Clean) and the other settings with a pre-processor, without a pre-processor,
all the small perturbations/noise contained in the input are fed into the foundation model,
which leads to the worst robust test accuracy of only 32.78\% for PGD-10, 19.14\% for PGD-20 and 2.11\% for AutoAttack.
Consistent with our theorem, the na\"ive approach to adversarial auto-encoder, \vanilla,
only provides a slight improvement to robust accuracy compared to \none.
Besides, comparing (\none, Clean) with (\none, Robust), when training the last linear layer using adversarial training,
the adversarial test accuracy can be significantly improved to 73.46\%, 65.76\% and 45.28\% for PGD-10, PGD-20 and AutoAttack.
However, these accuracy values are still much lower than the expensive, robust fine-tuning setting of around 96\%.


For the proposed method \contrastive,
the weight for the adversarial contrastive loss ($\lambda$) controls the clean and robust accuracy trade-off as expected:
The higher the $\lambda$ is, the more we favor robust accuracy over clean accuracy.
Recall that the no pre-processor (\none) robust training with and without fine-tuning achieves around 95\% and 70\% robust accuracy, respectively.
When taking the best $\lambda=10$ and training the linear classification layer using clean training,
robust accuracy of \contrastive\ outperforms robust training without fine-tuning by upto 48\% with less than 1\% degradation in clean accuracy.
When the linear layer of \contrastive\ ($\lambda=10$) is trained  with robust training,
the robust accuracy can be further enhanced to 89.04\% for PGD-10, 87.03\% for PGD-20 and 93.39\% for AutoAttack,
Lastly, our proposed method significantly improves the robustness of the whole system by merely compromising around 3\% in the natural test accuracy.

On the other hand, for the other pre-processors, to compare \vanilla\ with \vae\ and \arae,
when the last linear layer is trained using clean training,
the adversarial test accuracy can be significantly enhanced under both PGD and AutoAttack.
Although these two methods are not specifically designed to remove adversarial attacks,
since they still perform denoising, they are expected to be robust to adversarial attacks to some extent.
This aligns with the literature that random smoothing can defend against adversarial attack~\citep{cao2017mitigating,liu2018towards,cohen2019certified}.
However, the robustness is not as strong as models trained from adversarial training,
i.e., the adversarial test accuracies are much lower than 89.04\% and 87.03\% achieved by our proposed method \contrastive.
We also try LoRA as a surrogate for faster fine-tuning, but the performance is as similar as \none\ only.



\subsection{Multi-class Experiments}
\label{sec:multi-class}

\begin{table*}[]
\vspace{-0.1in}
    \centering
    \small
    \begin{minipage}{0.45\linewidth}
        \centering
        \resizebox{\linewidth}{!}{\begin{tabular}{l|c@{\hskip 5pt}c@{\hskip 5pt}c@{\hskip 5pt}c}
            \toprule
            Pre-processor & Natural        & PGD-10                                & PGD-20                                & AutoAttack                            \\
            \midrule
            \midrule
            \none         & \textbf{91.56} & \phantom{0}2.83{\scriptsize$\pm$0.33} & \phantom{0}1.19{\scriptsize$\pm$0.21} & \phantom{0}5.08{\scriptsize$\pm$0.44} \\
            \quad + LoRA  & 88.10          & \phantom{0}3.90{\scriptsize$\pm$0.37} & \phantom{0}1.78{\scriptsize$\pm$0.26} & ---                                   \\
            \vae          & 66.62          & 14.10{\scriptsize$\pm$0.68}           & \phantom{0}7.70{\scriptsize$\pm$0.52} & 15.27{\scriptsize$\pm$0.70}           \\
            \vanilla      & 90.93          & 10.04{\scriptsize$\pm$0.61}           & \phantom{0}4.95{\scriptsize$\pm$0.42} & \phantom{0}6.93{\scriptsize$\pm$0.51} \\
            \arae         & 87.96          & 18.38{\scriptsize$\pm$0.76}           & \phantom{0}7.73{\scriptsize$\pm$0.51} & 13.44{\scriptsize$\pm$0.65}           \\
            \midrule
            \contrastive  & 79.31          & \textbf{47.99{\scriptsize$\pm$0.98}}  & \textbf{40.31{\scriptsize$\pm$0.99}}  & \textbf{66.05{\scriptsize$\pm$0.95}}  \\
            \bottomrule
        \end{tabular}}
        \subcaption{\cifarten}
    \end{minipage}\hfill
    \begin{minipage}{0.45\linewidth}
        \centering
        \resizebox{\linewidth}{!}{\begin{tabular}{l|c@{\hskip 5pt}c@{\hskip 5pt}c@{\hskip 5pt}c}
            \toprule
            Pre-processor & Natural        & PGD-10                               & PGD-20                               & AutoAttack                            \\
            \midrule
            \midrule
            \none         & 96.89          & 30.97{\scriptsize$\pm$1.44}          & 21.57{\scriptsize$\pm$1.26}          & \phantom{0}1.69{\scriptsize$\pm$0.42} \\
            \quad + LoRA  & \textbf{96.97} & 33.28{\scriptsize$\pm$1.47}          & 21.43{\scriptsize$\pm$1.29}          & ---                                   \\
            \vae          & 96.57          & \textbf{59.91{\scriptsize$\pm$1.51}} & 45.31{\scriptsize$\pm$1.57}          & 13.03{\scriptsize$\pm$1.10}           \\
            \vanilla      & 92.30          & 54.55{\scriptsize$\pm$1.52}          & 47.88{\scriptsize$\pm$1.58}          & 30.62{\scriptsize$\pm$1.45}           \\
            \arae         & 93.32          & 57.32{\scriptsize$\pm$1.55}          & 53.69{\scriptsize$\pm$1.56}          & 34.41{\scriptsize$\pm$1.46}           \\
            \midrule
            \contrastive  & 83.84          & 58.15{\scriptsize$\pm$1.59}          & \textbf{54.80{\scriptsize$\pm$1.49}} & \textbf{60.68{\scriptsize$\pm$1.54}}  \\
            \bottomrule
        \end{tabular}}
        \subcaption{\imagenette}
    \end{minipage}

    \begin{minipage}{0.45\linewidth}
        \centering
       \resizebox{\linewidth}{!}{ \begin{tabular}{l|cccc}
            \toprule
            Pre-processor & Natural        & PGD-10                               & PGD-20                                \\
            \midrule
            \midrule
            \none         & \textbf{99.85} & 33.60{\scriptsize$\pm$0.35}          & 30.33{\scriptsize$\pm$0.36}           \\
            \vae          & 88.09          & 14.17{\scriptsize$\pm$0.26}          & \phantom{0}8.27{\scriptsize$\pm$0.20} \\
            \vanilla      & 99.79          & 93.43{\scriptsize$\pm$0.18}          & 67.79{\scriptsize$\pm$0.34}           \\
            \arae         & 99.21          & 59.08{\scriptsize$\pm$0.35}          & 43.32{\scriptsize$\pm$0.37}           \\
            \midrule
            \contrastive  & 99.78          & \textbf{95.81{\scriptsize$\pm$0.14}} & \textbf{95.22{\scriptsize$\pm$0.16}}  \\
            \bottomrule
        \end{tabular}}
        \subcaption{\svhn}
    \end{minipage}\hfill
    \begin{minipage}{0.45\linewidth}
        \centering
        \resizebox{\linewidth}{!}{\begin{tabular}{l|c@{\hskip 5pt}c@{\hskip 5pt}c@{\hskip 5pt}c}
            \toprule
            Pre-processor & Natural        & PGD-10                                & PGD-20                                & AutoAttack                            \\
            \midrule
            \midrule
            \none         & \textbf{81.11} & \phantom{0}6.15{\scriptsize$\pm$0.47} & \phantom{0}4.39{\scriptsize$\pm$0.41} & \phantom{0}9.79{\scriptsize$\pm$0.58} \\
            \vae          & 44.21          & \phantom{0}6.67{\scriptsize$\pm$0.50} & \phantom{0}3.82{\scriptsize$\pm$0.38} & 13.76{\scriptsize$\pm$0.67}           \\
            \vanilla      & 80.09          & 58.18{\scriptsize$\pm$0.99}  & 55.94{\scriptsize$\pm$0.94}  & 18.48{\scriptsize$\pm$0.74}           \\
            \arae         & 76.02          & 50.28{\scriptsize$\pm$1.01}           & 47.86{\scriptsize$\pm$0.96}           & \textbf{23.54{\scriptsize$\pm$0.84}}  \\
            \midrule
            \contrastive  & 80.14         
            & \textbf{58.26{\scriptsize$\pm$1.00}}         & \textbf{56.59{\scriptsize$\pm$0.94} }
            & 18.25{\scriptsize$\pm$0.75} \\
            \bottomrule
        \end{tabular}}
        \subcaption{\cifarhund}
    \end{minipage}
    \begin{minipage}{0.45\linewidth}
        \centering
        \resizebox{\linewidth}{!}{\begin{tabular}{l|ccc}
            \toprule
            Pre-processor & Natural & PGD-10                                & PGD-20                                \\
            \midrule
            \midrule
            Identity      & 67.20   & \phantom{0}8.39{\scriptsize$\pm$0.54} & \phantom{0}7.13{\scriptsize$\pm$0.52} \\
            VAE           & 54.43   & 15.42{\scriptsize$\pm$0.72}           & \phantom{0}9.97{\scriptsize$\pm$0.59} \\
            Vanilla       & 67.23   & 13.07{\scriptsize$\pm$0.68}           & 10.70{\scriptsize$\pm$0.60}           \\
            ARAE          & 67.34   & 19.82{\scriptsize$\pm$0.79}           & 16.64{\scriptsize$\pm$0.72}           \\
            \midrule
            CRoPD         & 65.84   & \textbf{22.16{\scriptsize$\pm$0.82}}  & \textbf{18.73{\scriptsize$\pm$0.78}}  \\
            \bottomrule
        \end{tabular}}
        \subcaption{\imagenettiny}
    \end{minipage}
    \caption{
        Clean and robust accuracy comparison of various pre-processors across \cifarten, \imagenette, \svhn, \cifarhund\ and \imagenettiny.
        The results are evaluated under clean conditions and against adversarial attacks (PGD-10, PGD-20, and AutoAttack).
    }
    \label{tab:main-results}\vspace{-0.2in}
\end{table*}

We use full datasets for multi-class tasks,
including \cifarten, \imagenette, \svhn, \cifarhund, and \imagenettiny.
The results for all five datasets are summarized in Table~\ref{tab:main-results}.
Similar to \cifartwo, \contrastive\ generally has the largest improvement on the adversarial test accuracy compared to other methods,
especially for \cifarten\ (\(2.83\%\) to \(47.99\%\)).

Besides, although in some scenarios,
e.g., (\contrastive\ vs. \vae\ for \imagenette), (\contrastive\ vs. \vanilla\ for \svhn\ and \cifarhund),
some other methods achieve similar performance to \contrastive\ under PGD-10,
\contrastive\ achieves better robustness under PGD-20 and AutoAttack in general.
The similar performance of the other methods compared to \contrastive\ is also caused by special issues:
For \imagenette, \vae\ is originally trained using a superset of \imagenette,
so it is not surprising that it attains the highest performance under weak attacks;
For \cifarhund, the inflated number of classes requires more features to be retained.
Combined with small input dimension, the impact of the contrastive loss is weakened,
resulting in a similar performance between \vanilla\ and \contrastive.
This is further validated by results on \imagenettiny,
a dataset with even larger number of classes but much larger input dimension,
where \contrastive\ outperforms \vanilla\ and \arae\ by a large margin.



\section{Conclusion}\label{sec:conclusion}
In this paper, observing the computation challenge in fine-tuning a foundation model using adversarial training, we examined the role of adversarial contrastive learning to seek a strategy where downstream robustness can be obtained without using adversarial training on the foundation model.
We theoretically showed that the downstream adversarial loss can be upper-bounded by a combination of the downstream clean loss and the adversarial contrastive loss,
which implies that if a data pre-processor can result in a small adversarial contrastive loss, the robustness of the whole system can be improved. This inspired us to develop CRoPD to pre-process image data. Experiments demonstrated the effectiveness of CRoPD.

\noindent\textbf{Limitations.} While our paper mainly considers image datasets, the overall idea is applicable to other modalities, e.g., language data. However, for different data types, the corresponding contrastive learning algorithm is different, and additional investigations are needed to develop the robust pre-processor.

\bibliographystyle{unsrtnat}
\bibliography{custom,aaai25}

\appendix
\onecolumn

\section{Transfer Ability}
\label{sec:transfer}

\begin{table*}[!htbp]
    \centering
    \small
    \begin{tabular}{ccl|ccc}
        \toprule
        Source                      & Target                      & Pre-processor & Natural        & PGD-10                   & PGD-20                   \\
        \midrule
        \midrule
        \multirow{3}{*}{\cifarten}  & \multirow{3}{*}{\cifarhund} & \none         & \textbf{81.11} & \phantom{0}6.15$\pm$0.47 & \phantom{0}4.39$\pm$0.41 \\
                                    &                             & \arae         & 75.62          & 49.38$\pm$1.00           & 46.03$\pm$0.99           \\
                                    &                             & \contrastive  & 78.43          & \textbf{56.41$\pm$0.98}  & \textbf{55.16$\pm$0.97}  \\
        \midrule
        \multirow{3}{*}{\cifarhund} & \multirow{3}{*}{\cifarten}  & \none         & \textbf{91.56} & \phantom{0}2.83$\pm$0.33 & \phantom{0}1.19$\pm$0.21 \\
                                    &                             & \arae         & 88.04          & 13.34$\pm$0.65           & \phantom{0}4.73$\pm$0.40 \\
                                    &                             & \contrastive  & 86.46          & \textbf{21.29$\pm$0.80}  & \textbf{10.84$\pm$0.62}  \\
        \bottomrule
    \end{tabular}

    \caption{
        Transfer learning results between \cifarten\ and \cifarhund\ datasets. The table shows the clean and adversarial accuracy for each pre-processor.
        The best values in each column are highlighted.
        The pre-processor is trained on the source dataset and applied to reconstruct and classify images in the target dataset.
        $\lambda = 1$ is used when targeting \cifarhund, as the foundation model is fine-tuned and more robust,
        and we set $\lambda = 6$ for the less robust \cifarten\ foundation model.
        Compared to \none\ and \arae, \contrastive\ attained meaningful features and significantly improved robust accuracies.\vspace{-0.1in}
    }
    \label{tab:transfer-results}
\end{table*}


In this experiment, we use \cifarten/\cifarhund\ to train the pre-processor and evaluate the adversarial test accuracy on \cifarhund/\cifarten\ respectively\footnote{Same as previous settings, we fine-tune the foundation model for \cifarhund\ and leave \cifarten\ foundation model as is.}. When leveraging contrastive learning, the auto-encoder is expected to comprehensively learn all features from the data, thus we expect that \contrastive\ can also transfer across datasets, better than other methods.
The results are in Table~\ref{tab:transfer-results}. From Table~\ref{tab:transfer-results}, both \arae\ and \contrastive\ learn meaningful features for reconstruction,
providing excellent natural performance similar to the upper bound by \none.
Meanwhile, our \contrastive\ manages to exceed the robust performance of \arae\ by up to 10\% across all settings.



\section{Extra Experimental Details}
\label{app-sec:experiment-setup}
We describe our detailed experimental settings in this section.
All experiments in this paper are conducted using a private cluster with Intel(R) Xeon(R) Gold 6326 CPU @ 2.90GHz and NVIDIA A100 GPUs.
We report the confidence intervals by bootstrapping the results equal to the full size of the test set for $1\,000$ repeats.

\paragraph{Foundation Model and Downstream Task}
For all experiments, we use the \huggingface\ Transformers \citep{wolf2020transformers} to load a
pre-trained large \vitmae\ \citep{he2021masked} as the foundation model.
We consider different scenarios of using the foundation model.
First, for some datasets, \eg, \cifartwo, \cifarten, and \imagenette, their distribution lines up with the pre-trained foundation model.
Thus, we freeze the foundation model as-is for these datasets to achieve a low-cost but efficient prediction.
Second, we fine-tune the pre-trained foundation model using clean training for better alignments for \cifarhund\ and \svhn. 
Third, for \cifartwo, we use adversarial training to fine-tune the foundation model with
Projected Gradient Descent (PGD) \citep{madry2017towards}, which serves as a benchmark and is expected to be the most robust baseline.

Finally, to perform classification for all the scenarios above, we instantiate and train a linear layer that maps features of the foundation model to labels.

\paragraph{Pre-processor}
We use a modified version of \vitmae\ built on top of the codebase by \citet{zhou2023eliminating} for \contrastive, \vanilla, and \arae.
Our \vitmae\ applies deterministic masking with a mask rate of $50\%$ and consists of a $8$-layered encoder and $2$-layered decoder,
where the attentions have $3$-heads each.
For datasets with lower resolution, we set the patch size to $2$ with a $192$-dimension embedding,
and for \imagenette\ we use a patch size of $14$ with a $768$-dimension embedding size.
Before feeding the latent embeddings to the projector, we first perform an average pooling with a pooling factor of $8$.
The resulting tensor is then flattened and fed into a two-layered projector with $2048$ hidden size and $128$ output size.
Table~\ref{tab:auto-encoder-settings} shows the parameters used in setting up \contrastive, \arae\ and \vae.
For reconstruction, we use standard mean squared error instead of binary cross entropy for better performance.

We apply data augmentation (Table~\ref{tab:auto-encoder-augmentation-settings}) and/or adversarial perturbation to the inputs while training the auto-encoders.
For the adversarial perturbation, we use the Fast Gradient Sign Method (FGSM) \cite{goodfellow2014explaining} with
$\ell_\inf$-norm with a max perturbation of $8/255$ and $4/255$ for low- and high-resolution datasets.
Finally, the remaining optimization settings are shown in Table~\ref{tab:auto-encoder-training-settings}.

\begin{table}[!htbp]
    \centering
    \small
    \begin{tabular}{ll}
        \toprule
        \multicolumn{2}{c}{\textbf{Pre-processor settings}} \\
        \midrule
        \midrule
        \textbf{\arae} & $\gamma = 0.1$ \\
        \midrule
            \multirow{4}{*}{\textbf{\contrastive}} 
        & $\lambda = 10$ for \cifarten \\ \cmidrule{2-2}
        & $\lambda = 1$ for \svhn, \imagenettiny \\ \cmidrule{2-2}
        & $\lambda = 0.2$ for \imagenette \\ \cmidrule{2-2}
        & $\lambda = 0.01$ for \cifarhund \\ 
    \midrule
    \textbf{\vae} & diffusers.AutoencoderKL \\ 
                          & stabilityai/sd-vae-ft-mse-original \\ 
            \midrule
        \bottomrule
    \end{tabular}
    \caption{Pre-processor settings for different datasets.}
    \label{tab:auto-encoder-settings}
\end{table}

\begin{table}[!htbp]
    \centering
    \small
    \begin{tabular}{l}
        \toprule
        \multicolumn{1}{c}{\textbf{Augmentation settings}} \\
        \midrule
        \midrule
        \multicolumn{1}{c}{\textbf{\cifartwo, \cifarten, \svhn, \cifarhund}} \\
        \midrule
        RandomResizedCrop(32) \\
        RandomHorizontalFlip() \\
        RandomApply([ColorJitter(0.4, 0.4, 0.4, 0.1)], p=0.8) \\
        RandomGrayscale(p=0.2) \\
        \midrule
        \multicolumn{1}{c}{\textbf{\imagenette}} \\
        \midrule
        RandomResizedCrop(224) \\
        RandomHorizontalFlip() \\
        RandomApply([ColorJitter(0.4, 0.4, 0.4, 0.1)], p=0.8) \\
        RandomGrayscale(p=0.2) \\
        \bottomrule
    \end{tabular}
    \caption{Augmentation settings for \contrastive\ and \vanilla. Note that \arae\ does not use augmented images.}
    \label{tab:auto-encoder-augmentation-settings}
\end{table}

\begin{table}[!htbp]
    \centering
    \small
    \begin{tabular}{ll}
        \toprule
        \multicolumn{2}{c}{\textbf{Optimization settings}} \\
        \midrule
        \midrule
        \textbf{Optimizer} & AdamW \\
        \midrule
        \textbf{Learning rate} & $1.5 \times 10^{-4}$ \\
        \midrule
        \textbf{Weight decay} & $5 \times 10^{-2}$ \\
        \midrule
        \textbf{Warmup epochs} & $20$ \\
        \midrule
        \textbf{Scheduler} & CosineAnnealingLR \\
        \midrule
        \multirow{7}{*}{\textbf{Batch size}} & \cifartwo: $32$ (until epoch $150$) \\
                                             & $256$ (after epoch $150$) \\ \cmidrule{2-2}
                                            & \cifarten: $32$ \\ \cmidrule{2-2}
                                            & \svhn: $32$ \\ \cmidrule{2-2}
                                            & \cifarhund: $96$ \\ \cmidrule{2-2}
                                            & \imagenette: $96$ \\ \bottomrule
        \multirow{5}{*}{\textbf{Max epochs}} & \cifartwo: $400$ epochs \\ \cmidrule{2-2}
                                         & \cifarten: $400$ epochs \\ \cmidrule{2-2}
                                         & \svhn: $400$ epochs \\ \cmidrule{2-2}
                                         & \cifarhund: $150$ epochs \\ \cmidrule{2-2}
                                         & \imagenette: $150$ epochs \\ \midrule
    \end{tabular}
    \caption{Optimization settings for training pre-processors on different datasets.}
    \label{tab:auto-encoder-training-settings}
\end{table}

\paragraph{Foundation Model and Linear Layer}

After training the pre-processors, we chain it with the foundation model (\texttt{facebook/vit-mae-large}) and train a linear classification layer.
As described before, we will also optionally fine-tune the foundation model.
Since the \vitmae\ expects $224 \times 224$ sized inputs, which mismatch with the image size of \cifarten, \svhn\ and \cifarhund,
we use the differentiable torch bi-linear interpolation to upscale image tensors from their original $32 \time 32$ shape.
We also apply the preprocessing normalization steps with parameters described by their configurations.
When training the linear classification layer (and optionally fine-tuning the foundation models),
we perform optimization until convergence or when it reaches max epochs.
For robust training, we use PGD-10 with a larger step size of $0.007$ to generate adversarial examples and mix them with natural samples.
The other detailed training parameters we used are shown in Table~\ref{tab:classification-training-settings}.

\begin{table}[!htbp]
    \centering
    \small
    \begin{tabular}{ll}
        \toprule
        \multicolumn{2}{c}{\textbf{Optimization settings}} \\
        \midrule
        \midrule
        \textbf{Optimizer} & AdamW \\
        \midrule
        \multirow{2}{*}{\textbf{Learning rate}} & $1 \times 10^{-2}$ (classification head only) \\ \cmidrule{2-2}
                                                & $1 \times 10^{-4}$ (fine-tuning foundation) \\
        \midrule
        \textbf{Batch size} & $64$ \\
        \midrule
        \textbf{Max epochs} & $150$ \\
        \midrule
        \textbf{LR scheduler} & multiply by $0.1$ after epochs $30$, $70$, $100$ \\
        \bottomrule
    \end{tabular}
    \caption{Optimization settings for training linear classification head with optional fine-tuning of the foundation model.}
    \label{tab:classification-training-settings}
\end{table}

\section{Additional Experimental Results}
\label{app-sec:experiment-results}

In this appendix, we show some additional experimental results.

\subsection{Effect of Downstream Data Quality on Robustness}
\label{sec:ablation-data-quality}

We conduct an ablation study to investigate the effect of downstream data quality during data pre-processor training and fine-tuning on adversarial robustness.

\paragraph{Data generation.}We generate eight synthetic datasets of varying quality using a single
pretrained image diffusion model on CIFAR-10.
We employ the Denoising Diffusion Probabilistic Model (DDPM) \citep{ho2020denoising} as our data generator, adopting the pre-trained weights from the Hugging Face Diffusers repository \citep{google_ddpm_cifar10_32}.

We fix the dataset size to $50\,000$ and vary two sampling hyperparameters to control data quality:
(i)~the sampling strategy (either the default \emph{stochastic} reverse process schedule of DDPM or the deterministic sampling strategy of Denoising Diffusion Implicit Models (DDIM;~\citealp{song2020denoising})),
and (ii)~the number of denoising inference steps.
While default stochastic DDPM sampling injects diffusion noise throughout the reverse-time Markov transitions, DDIM utilizes a deterministic discretizations of the diffusion trajectory, which inherently affects data fidelity. Furthermore, fewer inference steps reduce computational costs but generally degrade sample fidelity.

Table~\ref{tab:metrics-robustness} presents various data quality metrics for the eight generated datasets.
The first column denotes each setup by its sampling strategy and number of denoising inference steps (e.g., Stochastic-250 denotes the default stochastic DDPM sampling strategy with 250 steps; Deterministic-100 denotes the DDIM sampling strategy with 100 steps).
These configurations yield a wide range of data quality: test complementarity from 4.75\% to 23.65\%, test coverage from 12.56\% to 25.33\%, mean Inception Score (IS) from 4.40 to 7.07, and Fr\'{e}chet Inception Distance (FID) from 9.46 to 76.47. All data quality metrics are computed following the definitions in~\citet{gowal2021improving}.

\begin{table}[h!]
\centering
\resizebox{\textwidth}{!}{%
\begin{tabular}{l ccc cc cc c}
\toprule
 & \multicolumn{3}{c}{\textsc{Complementarity}} & \multicolumn{2}{c}{\textsc{Coverage}} & \multicolumn{2}{c}{\textsc{Inception Metrics}} & \textsc{Robust} \\
\cmidrule(lr){2-4} \cmidrule(lr){5-6} \cmidrule(lr){7-8} \cmidrule(lr){9-9}
\textsc{Setup} & \textsc{Train} & \textsc{Test} & \textsc{Self} & \textsc{Train} & \textsc{Test} & IS $\uparrow$ & FID $\downarrow$ & \textsc{Accuracy} $\uparrow$ \\
\midrule
Deterministic-100  & 23.83\% & 23.65\% & 52.52\% & 25.48\% & 25.33\% & $\textbf{7.07} \pm 0.09$ &  \textbf{9.46} & 30.91\% \\
Deterministic-25   & 17.65\% & 17.89\% & 64.46\% & 22.85\% & 22.82\% & $6.50 \pm 0.05$ & 16.99 & 26.70\% \\
Stochastic-250     & 15.55\% & 14.55\% & 69.90\% & 20.07\% & 20.14\% & $6.13 \pm 0.05$ & 29.96 & 29.30\% \\
Stochastic-175     & 12.53\% & 12.42\% & 75.05\% & 18.50\% & 18.96\% & $5.87 \pm 0.08$ & 36.52 & \textbf{31.80}\% \\
Stochastic-150     &  9.89\% & 10.26\% & 79.85\% & 17.42\% & 17.73\% & $5.53 \pm 0.05$ & 44.70 &  9.49\% \\
Stochastic-120     &  7.41\% &  7.00\% & 85.59\% & 14.84\% & 15.15\% & $4.92 \pm 0.05$ & 61.41 &  8.10\% \\
Stochastic-110     &  5.96\% &  5.62\% & 88.42\% & 13.74\% & 13.75\% & $4.61 \pm 0.05$ & 69.47 &  7.37\% \\
Stochastic-100     &  4.87\% &  4.75\% & 90.38\% & 12.69\% & 12.56\% & $4.40 \pm 0.05$ & 76.47 &  7.40\% \\
\bottomrule
\addlinespace[6pt]
\end{tabular}%
}
\caption{Various data quality metrics for synthetic CIFAR-10 datasets generated 
under different setups, alongside robust accuracy evaluated on CIFAR-10 test images for downstream tasks. For data quality, we consider complementarity, coverage, and inception metrics: \textit{complementarity} measures the proportion of samples whose nearest neighbor in the Inception feature space lies in the training set, test set, or the sampled set itself; \textit{coverage} reports the proportion of unique nearest neighbors; \textit{Inception Score} (IS) estimates sample quality and diversity, while \textit{Fr\'{e}chet Inception Distance} (FID) measures distributional similarity to real data. 
For IS, we report means with standard deviations computed across the standard split-wise estimates. Robust accuracy is evaluated against PGD-10 attacks with $\mathcal{L}_{\infty}$ perturbation budget $\epsilon = 8/255$.}
\label{tab:metrics-robustness}
\vspace{-1.5em}
\end{table}

\paragraph{Training and robustness evaluation protocol.}
We apply CRoPD while varying only the synthetic pre-processor-training corpus among the eight CIFAR-10 setups in Table~\ref{tab:metrics-robustness}. For both downstream classifier fine-tuning and robust evaluation stages, we use the pretrained ViT-MAE-Large foundation, with weights initialized from the publicly released checkpoint on the Hugging Face Hub \citep{facebook_vit_mae_large}.

For each setup, \(\mathcal{D}_{\text{down}}\) denotes the corresponding unconditional synthetic CIFAR-10 corpus containing $50\,000$ samples. These images are generated by
class-unconditional DDPM sampling and therefore do not have generator-provided
CIFAR-10 class labels. Accordingly, we implement Line~2 of Algorithm~\ref{alg:contra} to train
the data pre-processor on $\mathcal{D}_{\text{down}}$ using the reconstruction
and label-free adversarial contrastive objective, with contrastive loss weight
$\lambda=10$.

Following Line~3 of Algorithm~\ref{alg:contra}, we use the standard real,
class-labeled CIFAR-10 training split as $\mathcal{D}_{\text{label}}$. After
pre-processor training, we fine-tune only the last linear classification head by minimizing the clean cross-entropy classification loss while freezing the data pre-processor and all previous layers of the foundation model.

We then evaluate robust accuracy on CIFAR-10 test images using PGD-10
\citep{madry2017towards} with an $\ell_\infty$ perturbation budget
$\epsilon=8/255$, a step size $\alpha=1.6/255$, and $10$ attack iterations.

\begin{figure}[h!]
    \centering
    \includegraphics[width=0.95\textwidth]{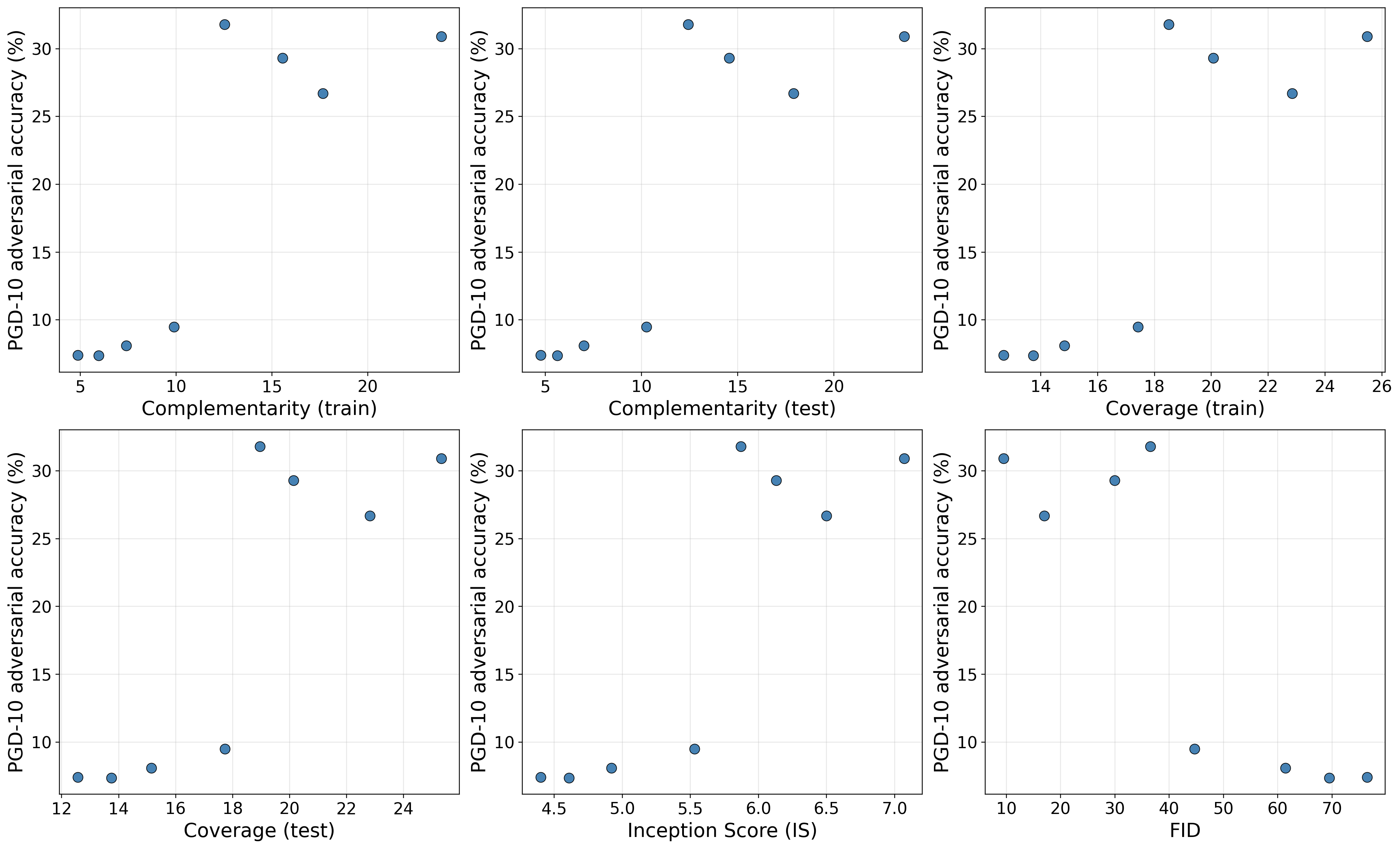}
    \caption{Scatter plots of data quality metrics for the synthetic CIFAR-10 datasets versus downstream adversarial robustness. The eight points correspond to the setups in Table~\ref{tab:metrics-robustness}. The coefficient of determination, $R^2$, ranges from 69.9\% to 75.9\% with p-values smaller than 0.01, indicating significant linear relationships between data quality and adversarial accuracy.}

    \label{fig:metrics-vs-robustness-scatter}
\end{figure}

\paragraph{Results.} The ablation study shows that, even in our pre-training and fine-tuning paradigm where the foundation model's pre-training data and its pre-trained weights remain fixed, downstream data quality is positively associated with downstream adversarial robustness. We report the robust accuracy in the last column of Table~\ref{tab:metrics-robustness} and visualize scatter plots of data quality metrics versus adversarial accuracy in Figure~\ref{fig:metrics-vs-robustness-scatter}. Across all metrics, adversarial accuracy increases from 7.37\% to 31.80\% as data quality improves (with lower FID, higher IS, and higher complementarity
and coverage). This positive linear relationship is statistically significant, with $R^2$ ranging from 69.9\% to 75.9\% with p-values smaller than 0.01.


\subsection{Comparison with Encoder-Only Adversarial Training}
\label{app-sec:encoder-only}

\begin{table}[ht]
    \centering
    \begin{tabular}{llcccc}
        \toprule
        \textbf{Dataset} & \textbf{Method} & \textbf{Clean}         & \textbf{PGD-10}        & \textbf{PGD-20}        & \textbf{Time}       \\
        \midrule
        \midrule
        \multirow{2}{*}{\cifarten}
                         & Encoder-Only    & 65.58{\small$\pm$0.94} & 28.08{\small$\pm$0.86} & 25.61{\small$\pm$0.86} & 315 sec * 50 epoch  \\
                         & \contrastive    & 79.31{\small$\pm$0.40} & 47.99{\small$\pm$0.98} & 40.31{\small$\pm$0.99} & 250 sec * 100 epoch \\
        \midrule
        \multirow{2}{*}{\imagenette}
                         & Encoder-Only    & 66.09{\small$\pm$1.47} & 41.95{\small$\pm$1.57} & 41.57{\small$\pm$1.55} & 310 sec * 90 epoch  \\
                         & \contrastive    & 83.84{\small$\pm$0.62} & 58.15{\small$\pm$1.59} & 54.80{\small$\pm$1.49} & 360 sec * 120 epoch \\
        \bottomrule
    \end{tabular}
    \caption{
        Comparison of encoder-only adversarial training and \contrastive\ under PGD attacks.
    }
    \label{tbl:encoder-only-results}
\end{table}

Table~\ref{tbl:encoder-only-results} compares two training strategies across the \cifarten\ and \imagenette\ datasets.
The ``Encoder-Only'' approach employs supervised adversarial training using only the encoder and the downstream dataset.
In contrast, \contrastive\ utilizes a robust auto-encoder (encoder+decoder) together with a frozen foundation model.
On both datasets, \contrastive\ yields consistently higher accuracy and adversarial robustness.
For example, on \cifarten, the clean accuracy improves from 65.58\% to 79.31\%,
while PGD-10 and PGD-20 robustness increase from 28.08 and 25.611\% to 47.991\% and 40.311\%, respectively.
A similar trend is observed for \imagenette.
Although \contrastive\ requires longer training times
(250 sec * 100 epoch for \cifarten\ and 360 sec * 120 epoch for \imagenette) compared to the Encoder-Only method,
this additional computational cost is offset by the significant performance gains achieved
by leveraging robust contrastive learning and the foundation model.

\subsection{Impact of sample size on performance}
\label{app-sec:dataset-size-ablation}

\begin{table}[ht]
    \centering
    \begin{tabular}{cc|ccc}
        \toprule
        Sample size & $\lambda$                 & \textbf{Clean}         & \textbf{PGD-10}        & \textbf{PGD-20}                  \\
        \midrule
        \midrule
        10\%        & \phantom{0}0.15           & 79.28{\small$\pm$0.82} & 11.92{\small$\pm$0.64} & \phantom{0}8.43{\small$\pm$0.55} \\
        20\%        & \phantom{0}1\phantom{.00} & 79.81{\small$\pm$0.79} & 16.01{\small$\pm$0.73} & \phantom{0}9.98{\small$\pm$0.58} \\
        50\%        & 10\phantom{.00}           & 81.85{\small$\pm$0.78} & 25.75{\small$\pm$0.84} & 14.33{\small$\pm$0.71}           \\
        100\%       & 10\phantom{.00}           & 79.31{\small$\pm$0.79} & 47.99{\small$\pm$0.98} & 40.31{\small$\pm$0.99}           \\
        \bottomrule
    \end{tabular}
    \caption{
        Ablation study of \contrastive\ on \cifarten\ by reducing the training set to 10\%, 20\%, 50\% and 100\%.
        We use different $\lambda$ values to balance the reconstruction and contrastive objectives given the size of the training set.
        \contrastive\ shows substantial gains with increased training data, underscoring its data efficiency and scalability.
    }
    \label{tab:ablation-cropd}
\end{table}

In this experiment, we hypothesis and test how sample size per class impacts \contrastive\ performance.
We consider a varying percentage of data available in the \cifarten\ dataset and the corresponding best hyperparameter $\lambda$.
Table~\ref{tab:ablation-cropd} indicates that, with limited per-class data (lower percentage of sample size),
the preprocessor struggles to learn meaningful representations.
And a smaller adversarial contrastive loss weight (e.g., $\lambda = 0.1$)
will better balance the reconstruction and contrastive objectives;
note that the \vanilla\ baseline corresponds to $\lambda = 0$.
They highlight the critical importance of sufficient per-class data in achieving high robustness,
which is consistent with our theoretical insights (Theorem~\ref{thm:downstream}).
With adequate data, \contrastive\ consistently outperforms reconstruction-only baseline (\vanilla)
in both clean accuracy and adversarial robustness.

%
%

\subsection{ROCL Trained Foundation Model}
\label{app-sec:rocl-results}

\begin{table*}[!htbp]
    \centering
    \small
    \begin{tabular}{l@{}c||ccc|ccc}
        \toprule
        Pre-processor               & Fine-tuning &                & Clean                                 &                                       &                & Robust                                &                                       \\
                                    & foundation  & Natural        & PGD-10                                & PGD-20                                & Natural        & PGD-10                                & PGD-20                                \\
        \midrule
        \midrule
        \none                       &             & \textbf{98.41} & \phantom{0}1.40{\scriptsize$\pm$0.53} & \phantom{0}0.74{\scriptsize$\pm$0.38} & \textbf{98.77} & \phantom{0}0.21{\scriptsize$\pm$0.19} & \phantom{0}0.00{\scriptsize$\pm$0.00} \\
        \vanilla                    &             & 98.31          & \phantom{0}1.00{\scriptsize$\pm$0.42} & \phantom{0}0.35{\scriptsize$\pm$0.25} & 96.08          & \phantom{0}9.15{\scriptsize$\pm$1.26} & \phantom{0}5.04{\scriptsize$\pm$0.96} \\
        \arae                       &             & 92.56          & 32.83{\scriptsize$\pm$1.98}           & 27.11{\scriptsize$\pm$2.01}           & 89.67          & 50.38{\scriptsize$\pm$2.11}           & 45.41{\scriptsize$\pm$2.11}           \\
        \midrule
        \contrastive, $\lambda=0.1$ &             & 98.26          & \phantom{0}2.61{\scriptsize$\pm$0.68} & \phantom{0}1.40{\scriptsize$\pm$0.51} & 96.34          & 13.83{\scriptsize$\pm$1.52}           & 11.37{\scriptsize$\pm$1.39}           \\
        \contrastive, $\lambda=1$   &             & 95.05          & \textbf{75.98{\scriptsize$\pm$1.86}}  & 70.42{\scriptsize$\pm$2.04}           & 94.62          & 79.48{\scriptsize$\pm$1.76}           & 75.66{\scriptsize$\pm$1.81}           \\
        \contrastive, $\lambda=10$  &             & 95.16          & 75.27{\scriptsize$\pm$1.94}           & \textbf{71.60{\scriptsize$\pm$1.96}}  & 94.91          & \textbf{79.59{\scriptsize$\pm$1.72}}  & \textbf{76.64{\scriptsize$\pm$1.80}}  \\
        \midrule
        \none                       & \checkmark  & 98.77          & \phantom{0}0.21{\scriptsize$\pm$0.19} & \phantom{0}0.00{\scriptsize$\pm$0.00} & 97.66          & 76.51{\scriptsize$\pm$1.87}           & 71.83{\scriptsize$\pm$2.00}           \\
        \bottomrule
    \end{tabular}
    \caption{
        Natural and robust performance comparison of different pre-processors on \cifartwo\ with a custom foundation model trained with \textit{RoCL}.
        The results are presented for both clean and robust trained (Robust) models, evaluated under clean conditions and against adversarial attacks (PGD-10 and PGD-20).
        The best values in each column, excluding fine-tuned models, are highlighted.
        All foundation model weights are frozen during training except for the fine-tuned models.
        \contrastive\ achieves competitive robust accuracies for both PGD-10 and PGD-20, demonstrating strong performance.
    }
    \label{tab:cifartwo-all-results-rocl}
\end{table*}

Table~\ref{tab:cifartwo-all-results-rocl} demonstrates another set of experiments on \cifartwo\ with a custom-trained model (using RoCL).
Consistent with Table~\ref{tab:cifartwo-all-results}, \contrastive\ out-performs the competing benchmark \arae\ by a large margin.
Surprisingly, \contrastive\ even outperforms the (\none, fine-tune) setting, though the gap is narrower.

\section{Computation Requirements}
We conduct all experiments using a private cluster with Intel(R) Xeon(R) Gold 6326 CPU @ 2.90GHz and NVIDIA A100 GPUs.
We execute our experiments using six cores, 48G Memory, and 1 GPU.
We document the estimated experiment runtime for \cifartwo\ experiments in Table~\ref{tab:computation-requirements}.
The total computation required by our method (\contrastive+Linear) can be much cheaper than that of robust fine-tuning.
In addition, we also report the parameter counts for each model in Table~\ref{tab:parameter-counts}.
Our method only requires training of a fraction of parameters compared to the foundation model employed.

\begin{table*}[!htbp]
    \centering
    \small
    \begin{tabular}{l|ll|ll}
        \toprule
        Dataset     & \contrastive\ (h) & Linear (h)  & Fine-tune (h) & Robust fine-tune (h)  \\
        \midrule
        \cifarhund  & 5.0          & 3.0       & 7.3   & 19.6  \\
        \imagenette & 9.6          & 1.7       & 3.6   & 22.3  \\
        \bottomrule
    \end{tabular}
    \caption{
        Estimated computation cost for training each model type on \cifarhund\ and \imagenette.
        The training time of our method includes the \contrastive\ and Linear columns,
        while the alternative adversarial fine-tuning can take drastically longer.
        Regular fine-tuning, though beneficial for robustness, can taker longer than the Linear to train due to
        slower convergence.
    }
    \label{tab:computation-requirements}
\end{table*}

\begin{table*}[!htbp]
    \centering
    \small
    \begin{tabular}{l|lll}
        \toprule
        Dataset type & \contrastive & Foundation & Fraction\\
        \midrule
        Small & $1.8 \times 10^{7}$ & $3.0 \times 10^{8}$ & \phantom{0}6.1\% \\
        Large & $8.3 \times 10^{7}$ & $3.0 \times 10^{8}$ & 27.6\%\\
        \bottomrule
    \end{tabular}
    \caption{
        Parameter counts of models for different datasets. Fraction shows the 
        Small and large dataset types refer to non-\imagenette\ and \imagenette.
        By incorporating a pre-processor, we drastically reduced the number of parameters to tune during training by more than 90\% and 70\%, respectively.
    }
    \label{tab:parameter-counts}
\end{table*}

\section{Proofs}\label{sec:appendix:proof}

\newcommand{\ouradd}[1]{{\color{blue}#1}}

\newcommand{\questionable}[1]{{\color{red}#1}}

\paragraph{Outline of the proof of Theorem \ref{thm:downstream}}

We now outline the key steps in proving Theorem \ref{thm:downstream} below, with detailed derivations provided in the appendix.

We begin by applying Jensen's inequality to the expectation of the negative log probability of the model's prediction \(
\mathbb{E}_{q(x,y)} \left[ -\log \mathbb{E}_{p_\theta(z \mid a(x))} \left[ \hat{p}_T(y \mid z) \right] \right]\), {where \( a(x) = \arg\max_{\xadv \in \mathcal{A}(x)} \log \mathbb{E}_{p_\theta(z \mid \xadv)}[\hat{p}_T(y\mid z)], \forall x  \),} transforming the inner expectation into a more tractable form as \(\mathbb{E}_{q(x,y)} \left[ \mathbb{E}_{p_\theta(z \mid a(x))} \left[ -\log \hat{p}_T(y \mid z) \right] \right]\). Next, we express the difference between the expected values of the loss function under the original distribution $p_{\theta}(z|x)$ and the adversarial distribution $p_{\theta}(z|a(x))$ using the Kantorovich-Rubinstein duality  
{\small
\begin{align}
    &\mathbb{E}_{q(x,y)} \left[ \mathbb{E}_{p_\theta(z \mid a(x))} \left[ -\log \hat{p}_T(y \mid z) \right] - \mathbb{E}_{p_\theta(z \mid x)} \left[ -\log \hat{p}_T(y \mid z) \right] \right] \notag \\
    &\quad \leq M_C \cdot \mathbb{E}_{p^T(x,y)} \left[ W\left( p_\theta(z \mid a(x)), \, p_\theta(z \mid x) \right) \right].
\end{align}
}
This allows us to quantify the deviation of the adversarial distribution from the original one in latent space via the Wasserstein distance.

Assuming accurate feature reconstruction by the foundation model's decoder \(\fde(z) \approx x\) where \(z  = \fen(x)\), and robust latent feature production by the encoder, such that \(\| \fen(a(x)) - \fen(x) \| \) is small,
then the encoder produces deterministic and precise latent representations for given inputs, and the decoder reconstructs features accurately.
Consequently, the distributions \( p_{\theta}(z \mid x) \) and \( p_{\theta}(z\mid a(x)) \) are expected to be highly concentrated around \( \fen(x) \) and \( \fen(a(x)) \) with very small variance.
Given this concentration, it is reasonable to approximate these distributions as Dirac delta functions centered at \( \fen(x) \) and \( \fen(a(x)) \).
Respectively, we treat $p_{\theta}(z \mid a(x))$ and $p_{\theta}(z \mid x)$ as Dirac delta distributions centered on $\fen(a(x))$ and $\fen(x)$.
This allows us to bound the Wasserstein distance by the Euclidean distance $\|\fen(a(x)) - \fen(x)\|$ between the adversarial and original latent representations.

We then relate the Euclidean distance between $\fen(a(x))$ and $\fen(x)$ to the adversarial contrastive loss $\Lcon$ through a scaling constant $\kappa$.
This formalizes the relationship between minimizing contrastive loss and improving adversarial robustness \(\| \fen(a(x)) - \fen(x) \| \leq \kappa \cdot \Lcon \left( \fen(a(x)), \, \fen(x) \right)\).
Finally, combining these elements yields the main inequality of Theorem~\ref{thm:downstream},
bounding the adversarial loss in the downstream task by the clean downstream loss plus a scaled version of the adversarial contrastive loss.

This proof leverages the Lipschitz continuity of the decoder to bound the impact of adversarial perturbations in latent space and demonstrates how contrastive learning serves as a mechanism for enhancing robustness by aligning clean and adversarial representations.
Through this approach, we establish a theoretical connection between adversarial contrastive learning and the robustness of downstream tasks, providing a solid foundation for our proposed method.

\begin{proof}[Proof of Theorem \ref{thm:downstream}]

Let \( a(x) = \arg\max_{\xadv \in \mathcal{A}(x)} \log \mathbb{E}_{p_\theta(z \mid \xadv)}[\hat{p}_T(y\mid z)], \forall x  \),
where $\theta$ is the the parameters of $\fen$.
$z$ refers to the sample latent values generated by encoder distribution $\fen(x)$. 
\( \hat{p}_T(y\mid z) = P(\flast(\fpre(\fde(z)))= y)\) refers to the model's prediction of a certain class.
\( \mathcal{A}(x) \) represents the threat model.
We note that \( -\log \hat{p}_T(y \mid \fde(z)) \) was used in the Theorem~\ref{thm:downstream} in the main body instead of the standard
\( \hat{p}_T(y \mid z) \) to help better highlight the procedure of our mechanism.




Applying Jensen's inequality to the concave function \( \log(\cdot) \), we obtain
\[
-\log \mathbb{E}_{p_\theta(z \mid a(x))}\bigl[\hat{p}_T(y \mid z)\bigr]
\;\leq\;
\mathbb{E}_{p_\theta(z \mid a(x))}\bigl[-\log \hat{p}_T(y\mid z)\bigr].
\]

Consequentially, taking \( \mathbb{E}_{q(x,y)} \) on both sides, we further define the \(\Delta\) to the following form,
\begin{align*}
\Delta = &\mathbb{E}_{q(x,y)} \left[ \mathbb{E}_{p_\theta(z \mid a(x))}[-\log \hat{p}_T(y\mid z)] \right] 
- \mathbb{E}_{q(x,y)} \left[ \mathbb{E}_{p_\theta(z \mid x)}[-\log \hat{p}_T(y\mid z)] \right] \\
=&\mathbb{E}_{q(x,y)} \left[ \mathbb{E}_{p_\theta(z\mid a(x))}[-\log \hat{p}_T(y\mid z)] - \mathbb{E}_{p_\theta(z\mid x)}[-\log \hat{p}_T(y\mid z)] \right] .
\end{align*}

Consider Proposition~\ref{pro:decoder_robust} and Proposition~\ref{pro:down_lipschitz} below, there exists a constant \( M_C \) such that \( g(z) = \frac{-\log \hat{p}_T(y \mid z)}{M_C} \) is 1-Lipschitz.
This allows us to leverage the Kantorovich-Rubinstein duality to express the difference between the expectations under the distributions \( p_\theta(z \mid a(x)) \) and \( p_\theta(z \mid x) \) in terms of a Wasserstein distance:

\begin{eqnarray*}
&&\mathbb{E}_{q(x,y)} \left[ \mathbb{E}_{p_\theta(z \mid a(x))}[-\log \hat{p}_T(y \mid z)] 
- \mathbb{E}_{p_\theta(z \mid x)}[-\log \hat{p}_T(y \mid z)] \right] \\
&=& M_C \cdot \mathbb{E}_{q(x,y)} \left[ \mathbb{E}_{p_\theta(z\mid a(x))}[g(z)] 
- \mathbb{E}_{p_\theta(z\mid x)}[g(z)] \right] \\
&\leq& M_C \cdot \mathbb{E}_{q(x,y)} \left[ \sup_{\|h\|_L \leq 1} 
\left( \mathbb{E}_{p_\theta(z \mid a(x))}[h(z)] - \mathbb{E}_{p_\theta(z \mid x)}[h(z)] \right) \right] \\
&&\quad \text{(This arises when \(g(z)\) belongs to the set of 1-Lipschitz functions.)} \\
&=& M_C \cdot \mathbb{E}_{q(x,y)} \left[ W\left(p_\theta(z \mid a(x)), p_\theta(z \mid x)\right) \right] \\
&&\quad \text{(Follows from the Kantorovich–Rubinstein duality theorem.)}
\end{eqnarray*}

Consider the Wasserstein distance \( W\left(p_\theta(z \mid a(x)), p_\theta(z \mid x)\right) \), where \( W \) refers to the Wasserstein distance between the two distributions \( p_\theta(z \mid a(x)) \) and \( p_\theta(z \mid x) \). 




Since the decoder \( \fde \) is pre-trained and deterministic, the conditional distributions \( p_\theta(z \mid a(x)) \) and \( p_\theta(z \mid x) \) can be ideally considered Dirac delta distributions concentrated at points \( \fen(a(x)) \) and \( \fen(x) \), respectively. Under this assumption, the Wasserstein distance simplifies exactly to the Euclidean distance:
\[
W\bigl(p_\theta(z \mid a(x)), p_\theta(z \mid x)\bigr) 
= \|\fen(a(x)) - \fen(x)\|_2.
\]

Considering empirical distributions constructed from the dataset, we choose the transport plan matching each latent embedding \( z_i = \fen(x_i) \) to its corresponding adversarial embedding \( z_i' = \fen(a(x_i)) \). 
\[
W\bigl(p_\theta(z \mid a(x)), p_\theta(z \mid x)\bigr) 
\leq \frac{1}{N}\sum_{i=1}^N\|\fen(a(x_i)) - \fen(x_i)\|_2
=\mathbb{E}_{q(x,y)}\|\fen(a(x)) - \fen(x)\|_2.
\]
\paragraph{Evaluation of the conditional Wasserstein term.}
We now evaluate the term
\[
\mathbb{E}_{q(x,y)}
\left[
W\bigl(P_\theta(z\mid a(x)),P_\theta(z\mid x)\bigr)
\right].
\]
Following the deterministic-Dirac interpretation used in the original proof,
the conditional latent distributions can be ideally regarded as Dirac measures
concentrated at their corresponding encoder outputs:
\[
P_\theta(z\mid x)=\delta_{f_{\mathrm{en}}(x)},
\qquad
P_\theta(z\mid a(x))=\delta_{f_{\mathrm{en}}(a(x))}.
\]
Here, $\delta_u$ denotes the Dirac measure concentrated at the point $u$.

Therefore, for each $(x,y)\sim q(x,y)$,
\[
\begin{aligned}
W\bigl(P_\theta(z\mid a(x)),P_\theta(z\mid x)\bigr)
&=
W\bigl(
\delta_{f_{\mathrm{en}}(a(x))},
\delta_{f_{\mathrm{en}}(x)}
\bigr) \\
&=
\|f_{\mathrm{en}}(a(x))-f_{\mathrm{en}}(x)\|_2,
\end{aligned}
\]
where the last equality uses the fact that the Wasserstein distance between two
Dirac measures equals the distance between their support points. Taking
expectation with respect to $q(x,y)$ gives
\[
\mathbb{E}_{q(x,y)}
\left[
W\bigl(P_\theta(z\mid a(x)),P_\theta(z\mid x)\bigr)
\right]
=
\mathbb{E}_{q(x,y)}
\left[
\|f_{\mathrm{en}}(a(x))-f_{\mathrm{en}}(x)\|_2
\right].
\]
Thus, under the deterministic-Dirac interpretation, the conditional
Wasserstein term is exactly the expected encoder-level displacement between
the adversarial input $a(x)$ and the clean input $x$.

\paragraph{Supplementary discussion: an empirical lower-bound view.}
We provide an additional empirical interpretation of the conditional
Wasserstein term
\[
\mathbb{E}_{q(x,y)}
\left[
W_1\bigl(P_\theta(z\mid a(x)),P_\theta(z\mid x)\bigr)
\right].
\]
This argument does not serve as the direct proof of the upper bound in
Theorem~1. Instead, it shows that the Wasserstein distance between aggregated
empirical latent distributions gives a lower-bound-type quantity for the
expected sample-wise conditional Wasserstein distance.

Let
\[
S_N=\{(x_i,y_i)\}_{i=1}^N
\]
be i.i.d. samples drawn from $q(x,y)$. Define
\[
z_i=f_{\mathrm{en}}(x_i),
\qquad
z_i'=f_{\mathrm{en}}(a(x_i)).
\]
The corresponding empirical latent distributions are
\[
\widehat{\mu}_N
=
\frac1N\sum_{i=1}^N \delta_{z_i},
\qquad
\widehat{\nu}_N
=
\frac1N\sum_{i=1}^N \delta_{z_i'},
\]
where $\delta_u$ denotes the Dirac measure concentrated at the point $u$.

Using the Kantorovich formulation of the $1$-Wasserstein distance, we have
\[
W_1(\widehat{\nu}_N,\widehat{\mu}_N)
=
\min_{\pi\in\Pi(\widehat{\nu}_N,\widehat{\mu}_N)}
\sum_{i=1}^N\sum_{j=1}^N
\pi_{ij}\|z_i'-z_j\|_2,
\]
where
\[
\Pi(\widehat{\nu}_N,\widehat{\mu}_N)
=
\left\{
\pi\in\mathbb{R}_+^{N\times N}:
\sum_{j=1}^N\pi_{ij}=\frac1N,\ \forall i,
\quad
\sum_{i=1}^N\pi_{ij}=\frac1N,\ \forall j
\right\}.
\]
Now consider the diagonal coupling
\[
\pi_{ij}^*
=
\begin{cases}
\frac1N, & i=j,\\
0, & i\neq j.
\end{cases}
\]
This coupling is feasible since its row and column sums are both equal to
$1/N$. Therefore, by optimality of the Wasserstein coupling,
\[
\begin{aligned}
W_1(\widehat{\nu}_N,\widehat{\mu}_N)
&\le
\sum_{i=1}^N\sum_{j=1}^N
\pi_{ij}^*\|z_i'-z_j\|_2  \\
&=
\frac1N\sum_{i=1}^N
\|z_i'-z_i\|_2.
\end{aligned}
\]

Following the deterministic-Dirac interpretation,
the conditional latent distributions can be regarded as
\[
P_\theta(z\mid x_i)=\delta_{z_i},
\qquad
P_\theta(z\mid a(x_i))=\delta_{z_i'}.
\]
Hence, for each sample $(x_i,y_i)$,
\[
\begin{aligned}
W_1\bigl(P_\theta(z\mid a(x_i)),P_\theta(z\mid x_i)\bigr)
&=
W_1(\delta_{z_i'},\delta_{z_i})  \\
&=
\|z_i'-z_i\|_2.
\end{aligned}
\]
Substituting this identity into the previous inequality gives
\[
W_1(\widehat{\nu}_N,\widehat{\mu}_N)
\le
\frac1N\sum_{i=1}^N
W_1\bigl(P_\theta(z\mid a(x_i)),P_\theta(z\mid x_i)\bigr).
\]
Taking expectation with respect to the sampled dataset $S_N$ yields
\[
\mathbb{E}_{S_N}
\left[
W_1(\widehat{\nu}_N,\widehat{\mu}_N)
\right]
\le
\mathbb{E}_{S_N}
\left[
\frac1N\sum_{i=1}^N
W_1\bigl(P_\theta(z\mid a(x_i)),P_\theta(z\mid x_i)\bigr)
\right].
\]
Since $(x_i,y_i)$ are i.i.d. samples from $q(x,y)$, the right-hand side becomes
\[
\mathbb{E}_{S_N}
\left[
\frac1N\sum_{i=1}^N
W_1\bigl(P_\theta(z\mid a(x_i)),P_\theta(z\mid x_i)\bigr)
\right]
=
\mathbb{E}_{q(x,y)}
\left[
W_1\bigl(P_\theta(z\mid a(x)),P_\theta(z\mid x)\bigr)
\right].
\]
Therefore,
\[
\boxed{
\mathbb{E}_{S_N}
\left[
W_1(\widehat{\nu}_N,\widehat{\mu}_N)
\right]
\le
\mathbb{E}_{q(x,y)}
\left[
W_1\bigl(P_\theta(z\mid a(x)),P_\theta(z\mid x)\bigr)
\right].
}
\]

Thus, the Wasserstein distance between the aggregated empirical latent
distributions provides a lower-bound-type quantity for the expected
sample-wise conditional Wasserstein distance. Intuitively, the right-hand side
forces each clean latent representation $f_{\mathrm{en}}(x_i)$ to be compared
with its own adversarial counterpart $f_{\mathrm{en}}(a(x_i))$, whereas the
left-hand side allows optimal transport to re-match mass across different
samples. This re-matching flexibility can only reduce the transport cost.
Therefore, the gap
\[
\mathbb{E}_{q(x,y)}
\left[
W_1\bigl(P_\theta(z\mid a(x)),P_\theta(z\mid x)\bigr)
\right]
-
\mathbb{E}_{S_N}
\left[
W_1(\widehat{\nu}_N,\widehat{\mu}_N)
\right]
\ge 0
\]
may provide a useful window for future theoretical analysis of the difference
between sample-wise latent perturbation and aggregated latent distribution
shift.

Hence the conditional Wasserstein term admits the following two-sided
characterization:
\[
\mathbb{E}_{S_N}
\left[
W_1(\widehat{\nu}_N,\widehat{\mu}_N)
\right]
\le
\mathbb{E}_{q(x,y)}
\left[
W_1\bigl(P_\theta(z\mid a(x)),P_\theta(z\mid x)\bigr)
\right]
=
\mathbb{E}_{q(x,y)}
\left[
\|\fen(a(x))-\fen(x)\|_2
\right].
\]

Substituting this identity into the preceding inequality gives
\begin{align*}
&\mathbb{E}_{q(x,y)}
\left[
-\log
\mathbb{E}_{p_\theta(z\mid a(x))}
\bigl[\hat{p}_T(y\mid z)\bigr]
\right] \\
\leq\;&
\mathbb{E}_{q(x,y)}
\left[
\mathbb{E}_{p_\theta(z\mid x)}
\bigl[-\log \hat{p}_T(y\mid z)\bigr]
\right]
+
M_C\,
\mathbb{E}_{q(x,y)}
\left[
\|\fen(a(x))-\fen(x)\|_2
\right].
\end{align*}
This shows that the downstream adversarial loss can be controlled by the
expected encoder-level displacement between the adversarial input $a(x)$ and
the clean input $x$. Meanwhile, the empirical aggregated latent distribution
shift
\[
\mathbb{E}_{S_N}
\left[
W_1(\widehat{\nu}_N,\widehat{\mu}_N)
\right]
\]
provides a lower-bound for the same conditional Wasserstein term.





To connect the encoder-level displacement with the contrastive objective, instead of using a specific, predefined form of contrastive loss, we are leveraging the general idea behind contrastive learning. The central concept in contrastive learning is to maximize the similarity (minimize the distance) between pairs of similar samples (e.g., an original sample and its adversarial example). The following argument provided is an informal idea and a brief proof that by minimizing the distance between such pairs, we can effectively bound the adversarial loss.
Now, considering the relationship between this distance and contrastive loss, informally, design a contrastive loss \( \Lcon \) to minimize the distance between the encoded representations of similar pairs (e.g., a sample and its adversarial version):


Given this, we can connect the Euclidean distance \( \| \fen(a(x)) - \fen(x) \|_2 \) to the contrastive loss by introducing a scaling constant \( C \).
\[
\mathbb{E}_{q(x,y)} \| \fen(a(x)) - \fen(x) \|_2 \leq C \cdot \Lcon(\fen(x), \fen(a(x))).
\]
It is important to note that value of $C$ depends on the contrastive loss and learning schedule. Assuming that our choice of the loss function is a well-behaved model, then we can obtain a reasonable value of $C$ for bounding adversarial Euclidean distance.

By substituting this into the earlier inequality, we obtain:

\begin{align*}
&\mathbb{E}_{q(x,y)} \left[ - \log \mathbb{E}_{p_\theta(z\mid a(x))}[\hat{p}_T(y \mid z)] \right] \\
\leq&\mathbb{E}_{q(x,y)} \left[ \mathbb{E}_{p_\theta(z \mid x)}[-\log \hat{p}_T(y\mid z)] \right] 
+ M_C \cdot C \cdot \Lcon(\fen(x), \fen(a(x)))\\
\leq&\mathbb{E}_{q(x,y)} \left[ \mathbb{E}_{p_\theta(z \mid x)}[-\log \hat{p}_T(y\mid z)] \right] 
+ \kappa \sup_{\xadv \in \mathcal{A}(x)} \Lcon(\fen(\xadv), \fen(x)) .
\end{align*}
where \(\kappa = M_C \cdot C \).

Using the above bounds, in our study, we specifically consider the adversarial contrastive loss \( \Lcon \) defined as follows:
\[
\Lcon(\fen(\xadv), \fen(x)) = \mathbb{E}_{q(x_i,y_i)}\left[ \ellcon(\fen(\xadv_i), \fen(x_i)) \right],
\]
where provided with a set of clean examples $X = \{x_1, \ldots x_N\}$, each with an adversarial counterpart, $X^{\adv} = \{x^{\adv}_1, \ldots x^{\adv}_N\}$.
Without an explicit negative sampling process, for $x_i$, we now consider $\xadv_i$ as the similar example and the remaining samples in these sets, $X^{\negative} = X \bigcup X^{\adv} \backslash \{x_i, \xadv_i\}$ dissimilar and define the loss function as
\begin{align*}
\ellcon(\fen(\xadv_i), \fen(x_i)) = -\log \frac{\exp(\cossim(\fen(\xadv_i), \fen(x_i))/\tau)}{\sum_{x^{\negative} \in X^{\negative} } \exp(\cossim(\fen(x_i), \fen(x^{\negative})/\tau)}.
\end{align*}
In the above, the similarity between two vectors \( \mathbf{u} \) and \( \mathbf{v} \) is measured using cosine similarity and is defined as:
\[
\cossim(\mathbf{u}, \mathbf{v}) = \frac{\mathbf{u}^T \mathbf{v}}{\|\mathbf{u}\| \|\mathbf{v}\|},
\]
which measures the cosine of the angle between \( \mathbf{u} \) and \( \mathbf{v} \), normalized by their magnitudes.

Define for each $x$ the similarities
\[
s_0(x)=\cossim(\fen(\xadv),\fen(x)),
\quad
s_j(x)=\cossim(\fen(x_j^-),\fen(x))
\quad(j=1,\dots,M),
\]
where \(x_j^-\) is the negative example from \(X^{\negative}\). The per‐sample loss is
\[
\ellcon(\fen(\xadv), \fen(x))
=-\frac{s_0(x)}{\tau}
+\log\!\Bigl(\exp(s_0(x)/\tau)+\sum_{j=1}^M\exp(s_j(x)/\tau)\Bigr).
\]
Taking expectation over $x\sim q$, one obtains

\[
L_{\rm con}
=-\mathbb{E}_{q(x,y)}\!\bigl[s_0/\tau\bigr]
+\mathbb{E}_{q(x,y)}\Bigl[\log\!\sum_{k=0}^M e^{s_k/\tau}\Bigr].
\]
Rearrange:
\[
\mathbb{E}_{q(x,y)}[s_0]
=\tau\Bigl(\mathbb{E}_{q(x,y)}\bigl[\log\!\sum_{k=0}^M e^{s_k/\tau}\bigr]
L_{\rm con}\Bigr).
\]
By standard log‐sum‐exp bounds (or mutual‐information arguments),
\[
\mathbb{E}_{q(x,y)}\Bigl[\log\!\sum_{k=0}^M e^{s_k/\tau}\Bigr]
\;\ge\;
\ln(M+1).
\]
Hence
\[
\mathbb{E}_{q(x,y)}[\cos]
=\mathbb{E}_{q(x,y)}[s_0]
\;\ge\;
\tau\bigl(\ln(M+1)- L_{\rm con}\bigr).
\]

For any unit vectors $z,z'$,
\[
\|z'-z\|^2
=2 - 2\,{z'}^{T}z
=2 - 2\,\cos.
\]
Thus
\[
\|z'-z\|=\sqrt{2-2\cos}.
\]

Define the random variable $Y=2-2\cos\ge0$.  Then $\|z'-z\|=\sqrt{Y}$.  Since $f(y)=\sqrt y$ is concave on $[0,\infty)$,
Jensen’s inequality gives
\[
\mathbb{E}_{q(x,y)}[\|z'-z\|]
=\mathbb{E}_{q(x,y)}[\sqrt Y]
\;\le\;
\sqrt{\mathbb{E}_{q(x,y)}[Y]}
=\sqrt{2 - 2\,\mathbb{E}_{q(x,y)}[\cos]}.
\]

By choosing the identity coupling between the deterministic encodings of $x$ and $a(x)$, one shows
\[
W(p(z\mid a(x)),p(z\mid x))
\;\le\;
\mathbb{E}_{q(x,y)}\|z'-z\|
\;\le\;
\sqrt{2 - 2\,\mathbb{E}_{q(x,y)}[\cos]}
\;\le\;
\sqrt{2 - 2\,\tau(\ln(M+1)-\widehat L_{\rm con})},
\]

which generates the upper bound for \(\Delta\).
\end{proof}

\begin{proposition}\label{pro:decoder_robust}
    For the robust auto-encoder, assume we have a robust auto-encoder defined by an encoder \( \fen(x) \) and a decoder \( \fde(z) \), where \( z = \fen(x) \) represents the latent feature of the input \( x \). This auto-encoder is trained using a combination of reconstruction loss and adversarial contrastive loss. The reconstruction loss, \( \|\fde(\fen(x)) - x\|^2 \), ensures that the decoder \( \fde \) can accurately reconstruct the original input from the encoded features. The contrastive loss, \( \Lcon(\fen(x+\delta),\fen(x')) \), promotes robustness of the encoded features \( \fen(x) \) against adversarial perturbations \( \delta \). There exists a constant \( C \) such that \( \frac{\fde(z)}{C} \) is 1-Lipschitz.

\end{proposition}

\begin{proof}[Proof of Proposition \ref{pro:decoder_robust}]
    Given the training process of the robust auto-encoder, two key losses are employed: the reconstruction loss \( \|\fde(\fen(x)) - x\|^2 \) and the adversarial contrastive loss \( \Lcon(\fen(x+\delta),\fen(x')) \). These losses are designed to ensure that the encoder \( \fen(x) \) and decoder \( \fde(z) \) not only preserve the integrity of the original input \( x \) in the reconstruction but also maintain robustness against adversarial perturbations.


We denote the Lipschitz constant of our encoder as $L_{en}$.
Additionally, given that in a \vitmae, the representation space is typically larger than the input space,
we assume there exists a lower bound on the gradient between any samples in our dataset and within the threat model $\mathcal{A}$ of samples of our dataset.
Formally, we write this as

\[
l_{en} \|x_1 - x_2\| \leq \|z_1 - z_2\| \leq L_{en} \|x_1 - x_2\|,
\]

where \( z_1 = \fen(x_1) \) and \( z_2 = \fen(x_2) \) represent the latent features of the inputs \( x_1 \) and \( x_2 \), respectively.

Similarly, the reconstruction loss \( \|\fde(\fen(x)) - x\|^2 \) optimizes the decoder \( \fde \) to ensure accurate reconstruction of the input from the latent features. In most cases, it is safe to assume this optimization results in a finite-Lipschitz constant \( L_{rec} \) that describes how the output \( \fde(z) \) changes with respect to variations in the input \( x \):
\[
\|\fde(z_1) - \fde(z_2)\| \leq L_{rec} \|x_1 - x_2\|.
\]

Given the relationships established by these losses, we can now derive the Lipschitz constant between the encoded feature \( z \) and the decoder output \( \fde(z) \). By combining the encoder's and decoder’s Lipschitz constants, we obtain:

\[
\|\fde(z_1) - \fde(z_2)\| \leq L_{rec} \|x_1 - x_2\| \leq L_{rec} \cdot \frac{1}{l_{en}} \|z_1 - z_2\|.
\]

Thus, the Lipschitz constant for the decoder with respect to the encoded feature \( z \), denoted as \( L_{de}^{(z)} \), is:
\[
L_{de}^{(z)} = \frac{L_{rec}}{l_{en}}.
\]

Finally, by selecting the constant \( C = L_{de}^{(z)} \), we can ensure that the scaled decoder function \( g(z) = \frac{\fde(z)}{C} \) satisfies the 1-Lipschitz condition:

\[
|g(z_1) - g(z_2)| = \frac{1}{C} |\fde(z_1) - \fde(z_2)| \leq \|z_1 - z_2\|.
\]

This concludes the proof, demonstrating that such a constant \( C \) exists and that \( \frac{\fde(z)}{C} \) is indeed 1-Lipschitz, ensuring the stability and robustness of the auto-encoder.

\end{proof}

\begin{proposition}\label{pro:down_lipschitz}
    Assume \( -\log \hat{p}_T(y \mid \fde(z)) \leq M \) for all \( z \in \mathcal{Z}, y \in \mathcal{Y} \). If there exists a constant \( C \) such that \( \frac{\fde(z)}{C} \) is 1-Lipschitz, then there exists a constant \( M_C \) such that \( \frac{-\log \hat{p}_T(y \mid \fde(z))}{M_C} \) is 1-Lipschitz.

\end{proposition}

\begin{proof}[Proof of Proposition \ref{pro:down_lipschitz}]
Given that \( -\log \hat{p}_T(y \mid \fde(z)) \leq M \) for all \( z \in \mathcal{Z} \) and \( y \in \mathcal{Y} \), we know that the function \( -\log \hat{p}_T(y \mid \fde(z)) \) is bounded above by \( M \). Consequently, \( \hat{p}_T(y \mid \fde(z)) \) is bounded below by \( e^{-M} \).

Now, assume that there exists a constant \( C \) such that \( \frac{\fde(z)}{C} \) is 1-Lipschitz, which implies:

\[
\left|\frac{\fde(z_1)}{C} - \frac{\fde(z_2)}{C}\right| \leq \|z_1 - z_2\|, \quad \forall z_1, z_2 \in \mathcal{Z}.
\]

Next, consider the function \( -\log \hat{p}_T(y \mid \fde(z)) \). To show that \( \frac{-\log \hat{p}_T(y \mid \fde(z))}{M_C} \) is 1-Lipschitz for some constant \( M_C \), we recognize that the derivative of the log function is bounded by \( \frac{1}{\hat{p}_T(y \mid \fde(z))} \). Since \( \hat{p}_T(y \mid \fde(z)) \) is bounded below by \( e^{-M} \), the derivative is bounded by \( e^{M} \).

Thus, we define \( M_C \) as:
\[
M_C = C \times e^{M}.
\]

We then scale the function \( -\log \hat{p}_T(y \mid \fde(z)) \) by \( M_C \) and consider the difference:
\[
\left|\frac{-\log \hat{p}_T(y \mid \fde(z_1))}{M_C} - \frac{-\log \hat{p}_T(y \mid \fde(z_2))}{M_C}\right| = \frac{1}{M_C} \left|-\log \hat{p}_T(y \mid \fde(z_1)) + \log \hat{p}_T(y \mid \fde(z_2))\right|.
\]

Given the earlier bound on the derivative of the log function, we have:

\[
\left|\log \hat{p}_T(y \mid \fde(z_1)) - \log \hat{p}_T(y \mid \fde(z_2))\right| \leq e^{M} \cdot C \cdot \|z_1 - z_2\|.
\]

Substituting into our earlier equation, we get:

\[
\left|\frac{-\log \hat{p}_T(y \mid \fde(z_1))}{M_C} - \frac{-\log \hat{p}_T(y \mid \fde(z_2))}{M_C}\right| \leq \frac{e^{M} \cdot C}{M_C} \cdot \|z_1 - z_2\| = \frac{e^{M} \cdot C}{e^{M} \cdot C} \cdot \|z_1 - z_2\| = \|z_1 - z_2\|.
\]

Thus, \( \frac{-\log \hat{p}_T(y \mid \fde(z))}{M_C} \) is indeed 1-Lipschitz, which completes the proof.

\end{proof}

\begin{proof}[Proof of Proposition \ref{pro:adv loss}.] To illustrate an extreme scenario in proving Proposition\ref{pro:adv loss}, 
we construct a data generation model $(x,y)$ consisting of a discrete set 
$\{x_i\}$ with labels $y_i$, an identity auto-encoder 
$\bigl(f_{\mathrm{en}},f_{\mathrm{de}}\bigr)$ that achieves negligible 
reconstruction losses for both clean and adversarial inputs, 
and a classifier $\bigl(f_{\mathrm{pre}},f_{\mathrm{last}}\bigr)$ whose 
adversarial classification loss can be made arbitrarily large 
despite near-perfect reconstructions.

Step 1: Construct a data distribution $q(x,y)$.
We select a finite (or discrete) set of points 
\[
   \mathcal{X} \;=\; \{\,x_1,\dots,x_N\} \;\subset\;\mathbb{R}^d,
\]
each assigned a label $y_i \in \{1,\dots,K\}$.  
We assume these points are well-separated by at least $2\epsilon$, i.e., 
\[
   \|x_i - x_j\| \;\ge\; 2\epsilon 
   \quad
   \text{for all } i \neq j.
\]
We define a discrete distribution 
\[
   q(x,y)
   \;=\;
   \frac{1}{N}\sum_{i=1}^N \delta_{(x_i,y_i)}(x,y),
\]
so that each pair $(x_i,y_i)$ occurs with probability $\frac{1}{N}$.

\medskip
Step 2: Define an auto-encoder $(f_{\mathrm{en}}, f_{\mathrm{de}})$ with small reconstruction loss.
Consider the \emph{identity mapping}:
\[
   f_{\mathrm{en}}(x) \;=\; x,
   \quad
   f_{\mathrm{de}}(z) \;=\; z.
\]
Then for any input $x$ (including adversarially perturbed inputs $\xadv$),
\[
   f_{\mathrm{de}}\bigl(f_{\mathrm{en}}(x)\bigr) \;=\; x
   \quad\Longrightarrow\quad
   \bigl\| f_{\mathrm{de}}\bigl(f_{\mathrm{en}}(x)\bigr) - x \bigr\|^2 
   \;=\; 0.
\]
Hence for both clean data $x$ and any $\xadv$ with $\|\xadv - x\|\le \epsilon$, the reconstruction loss is \emph{exactly zero}, i.e.,
\[
   \mathbb{E}_{x}\Bigl[\|f_{\mathrm{de}}(f_{\mathrm{en}}(x)) - x\|^2 \Bigr] 
   \;=\; 0 , 
\]
and
\[
   \mathbb{E}_{x}\Bigl[\|f_{\mathrm{de}}(f_{\mathrm{en}}(\xadv)) - x\|^2 \Bigr] 
   \;=\; \epsilon^2=O(1/n) 
\]
if taking $\epsilon=O(1/\sqrt{n})$. Thus both reconstruction errors (clean and adversarial) can be made as small as desired.

\medskip
Step 3: Construct a brittle classifier $(f_{\mathrm{pre}}, f_{\mathrm{last}})$ with large adversarial loss.
Define $f_{\mathrm{pre}}$ to be the identity as well (or any feature extractor that yields $x$ effectively), so $f_{\mathrm{pre}}(x) = x$.  
Then let $f_{\mathrm{last}}\colon \mathbb{R}^d \to \Delta_K$ (the probability simplex) be chosen as follows:

\begin{enumerate}
\item For each clean point $x_i$, assign label $y_i$ with probability very close to 1:
\[
   \hat{p}_T\bigl(y_i \,\big|\; f_{\mathrm{last}}(\fpre(x_i))\bigr) = 1-\delta,
\]
for some arbitrary small constant $\delta>0$, then $-\log \hat{p}_T\bigl(y_i \mid f_{\mathrm{last}}(\fpre(x_i))\bigr) =O(\delta)$, ensuring near-zero loss on all clean $(x_i,y_i)$.
\item For any $\xadv$ satisfying $\|\xadv - x_i\|\le\epsilon$, assign \emph{zero} (or near-zero) probability to the label $y_i$, i.e.
\[
   \hat{p}_T\bigl(y_i \,\big|\; f_{\mathrm{last}}(\fpre(\xadv))\bigr)= O(\delta).
\]
Hence $-\log \hat{p}_T\bigl(y_i \mid f_{\mathrm{last}}(\fpre(\xadv))\bigr)\gg 0$, so the classification loss is made arbitrarily large on these adversarial perturbations.
\end{enumerate}

Because the $\epsilon$-balls around different $x_i$ do not intersect ($\|x_i - x_j\|\ge 2\epsilon$), we can define such a piecewise decision rule without conflict.  


\end{proof}

\section{License of Assets}\label{sec:license}
Table \ref{tab:license} summarizes the license of the assets used in this paper.
\begin{table}[!ht]
\centering
\caption{Licenses for datasets and code assets used in the paper.}\label{tab:license}
\resizebox{\linewidth}{!}{
\begin{tabular}{lll}
\toprule
\textbf{Asset} & \textbf{License} & \textbf{Link} \\
\midrule
CIFAR-10 & MIT License & \url{https://www.cs.toronto.edu/~kriz/cifar.html} \\
CIFAR-100 & MIT License & \url{https://www.cs.toronto.edu/~kriz/cifar.html} \\
SVHN & Unknown (often Public Domain) & \url{http://ufldl.stanford.edu/housenumbers/} \\
ImagenetTe & ImageNet Terms of Use & \url{https://image-net.org/} \\
ViT-MAE (facebook/vit-mae-large) & Apache 2.0 & \url{https://huggingface.co/facebook/vit-mae-large} \\
ARAE & No license specified & \url{https://github.com/rohban-lab/Salehi_submitted_2020} \\
\bottomrule
\end{tabular}}
\end{table}

\end{document}